\documentclass[10pt, conference, compsocconf]{IEEEtran}
\IEEEoverridecommandlockouts
\usepackage{amsmath,amssymb,amsfonts}
\usepackage{algorithmic}
\usepackage{graphicx}
\usepackage{textcomp}
\usepackage{xcolor}
\def\BibTeX{{\rm B\kern-.05em{\sc i\kern-.025em b}\kern-.08em
    T\kern-.1667em\lower.7ex\hbox{E}\kern-.125emX}}

\usepackage{balance}
\usepackage{inputenc}
\usepackage{amsmath,xcolor,bm}
\usepackage{enumitem,graphicx,algorithm,algorithmic}
\usepackage{microtype,subfigure,etoolbox}
\usepackage{tikz,color,float,booktabs}
\usepackage{breqn}
\usepackage{hyperref}
\hypersetup{colorlinks=true,linkcolor=blue,filecolor=magenta,urlcolor=cyan,citecolor=black}
\usepackage{mathtools}
\usepackage{booktabs}
\usepackage{caption}
\usepackage{colortbl}
\usepackage{xcolor}
\usepackage{amsthm}
\usepackage{amsmath}
\usepackage{xcolor}
\usepackage{multirow}
\usepackage{multicol}
\usepackage{array}
\usepackage{bbold}
\usepackage{makecell}
\usepackage{threeparttable}
\newcolumntype{P}[1]{>{\centering\arraybackslash}p{#1}}
\DeclarePairedDelimiter\ceil{\lceil}{\rceil}
\usepackage{xcolor}

\newtheorem{theorem}{Theorem}
\newtheorem{lemma}{Lemma}

\DeclareMathOperator*{\argmin}{arg\,min}
\newcommand{\sign}{\text{sign}}
\usepackage[numbers]{natbib}
\begin{document}

\title{Online AUC Optimization for Sparse High-Dimensional Datasets}

\author{
\IEEEauthorblockN{Baojian Zhou\IEEEauthorrefmark{1}, Yiming Ying\IEEEauthorrefmark{2}, Steven Skiena\IEEEauthorrefmark{1}}
\IEEEauthorblockA{\IEEEauthorrefmark{1}Department of Computer Science, Stony Brook University, Stony Brook, USA
\\\{baojian.zhou, skiena\}@cs.stonybrook.edu}
\IEEEauthorblockA{\IEEEauthorrefmark{2}Department of Mathematics and Statistics, University at Albany, Albany, USA
\\yying@albany.edu}}

\maketitle

\begin{abstract}
The Area Under the ROC Curve (AUC) is a widely used performance measure for imbalanced classification arising from many application domains where high-dimensional sparse data is abundant. In such cases,  each $d$ dimensional sample has only $k$ non-zero features with $k \ll d$, and data arrives sequentially in a streaming form. Current online AUC optimization algorithms have high per-iteration cost $\mathcal{O}(d)$ and usually produce non-sparse solutions in general, and hence are not suitable for handling the data challenge mentioned above. 
 
 In this paper, we aim to directly optimize the AUC score for high-dimensional sparse datasets under online learning setting and propose a new algorithm, \textsc{FTRL-AUC}. Our proposed algorithm can process data in an online fashion with a much cheaper per-iteration cost $\mathcal{O}(k)$, making it amenable for high-dimensional sparse streaming data analysis. Our new algorithmic design critically depends on a novel reformulation of the U-statistics AUC objective function as the empirical saddle point reformulation, and the innovative introduction of the ``lazy update'' rule so that the per-iteration complexity is dramatically reduced from $\mathcal{O}(d)$ to $\mathcal{O}(k)$. Furthermore, \textsc{FTRL-AUC} can inherently capture sparsity more effectively by applying a generalized Follow-The-Regularized-Leader (FTRL) framework. 
 
 Experiments on real-world datasets demonstrate that \textsc{FTRL-AUC} significantly improves both run time and model sparsity while achieving competitive AUC scores compared with the state-of-the-art methods. Comparison with the online learning method for logistic loss demonstrates that \textsc{FTRL-AUC} achieves higher AUC scores especially when datasets are imbalanced.
\end{abstract}


\section{Introduction}
\label{sect:introduction}

The Area Under the ROC Curve (AUC) score \cite{bradley1997use,fawcett2006introduction,hanley1982meaning} is a widely used performance metric to measure the quality of classifiers, particularly in  the problem of imbalanced classification where the size of one class is much larger than the other class. In such problems as online spam filtering~\cite{spam2007}, ad click prediction~\cite{mcmahan2013ad}, and identifying malicious URLs~\cite{ma2009identifying}, the datasets are not only imbalanced but also \textit{high-dimensional} and \textit{sparse}. Specifically,  such datasets are of very high dimension $d$, but the number of nonzero features $k$ in each training sample is far less than the total number of features $d$, i.e., $k \ll d$.  We consider this  type of online data, which arrives in a streaming fashion, requiring real-time training and predictions. Hence, it is of critical importance to develop  efficient online AUC optimization algorithms which can make prediction in a real time manner upon receiving  new high dimensional sparse data. 

Online machine learning (\textit{online learning})~\cite{shalev2012online,hazan2016introduction,cesa2006prediction} is a natural choice to deal with data in a real time manner as it can update the model sequentially. 
Most of the existing online learning algorithms~\cite{duchi2011adaptive,xiao2010dual,shalev2012online,hazan2016introduction} focus on the error rate (accuracy) where the objective function is the sum of pointwise losses over individual examples. Thus, they are not suitable for the problem of AUC maximization because the AUC objective function is the sum of pairwise losses over pairs of examples in the form of U-statistics~\cite{clemenccon2008ranking}. Recently, considerable work has been done to develop variants of stochastic (online) gradient descent algorithms for AUC maximization.  Specifically, the work of~\cite{Wang2012,kar2013generalization,zhao2011online,ying2016stochastic,ying2019spauc,liu2018fast,natole2018stochastic} proposes a variant of stochastic gradient descent (SGD) algorithms and the particular work~\cite{ying2019spauc,liu2018fast,natole2018stochastic} uses stochastic (online) proximal gradient algorithms to handle the sparse $\ell^1$-regularization.  

However, such online AUC optimization algorithms do not explore the structure of high-dimensional sparse data, and the per-iteration cost of at least $O(d)$ is expensive when $d$ is very large. Moreover, the produced AUC maximization models~\cite{ying2019spauc,liu2018fast,natole2018stochastic} updated by using $\ell^1$ regularization (constraints) do not produce sparse solutions. As such, the existing online algorithms cannot apply to  high-dimensional sparse data where the response prediction time is critically important~\cite{langford2009sparse,mcmahan2013ad}. 

Inspired by a generalized Follow-The-Regularized-Leader (FTRL) framework~\cite{mcmahan2010adaptive,mcmahan2011follow,mcmahan2013ad}, in this paper, we propose an online AUC optimization algorithm, namely \textsc{FTRL-AUC}, for high-dimensional sparse datasets.  Our new algorithm \textsc{FTRL-AUC} has three novel improvements: 
\begin{itemize}[leftmargin=*]
\item Our proposed algorithm \textsc{FTRL-AUC} can handle the streaming data in an online manner, i.e. updating the model parameter upon receiving each individual data point without the need of  pairing it with previous ones. Motivated by \cite{ying2016stochastic,natole2018stochastic}, we achieve this by reformulating  the original objective function of AUC maximization as an empirical saddle point formulation. 
\item  The per-iteration cost of \textsc{FTRL-AUC}, making full use of inherent sparsity of datasets, is $\mathcal{O}(k)$, which is much less than $\mathcal{O}(d)$ of the existing methods. It is challenging to directly obtain $\mathcal{O}(k)$ complexity because the gradient of the original loss at each iteration is not sparse. To overcome this obstacle, our key idea is to introduce a new surrogate loss and a novel ``lazy update" rule to update current positive score and negative score, and then apply the generalized FTRL framework. 

\item Finally, our experimental results demonstrate that \textsc{FTRL-AUC} significantly improves both run time and model sparsity compared to the state-of-the-art methods. We also compare our algorithm with other online learning algorithms for logistic loss and demonstrate that it achieves higher AUC scores especially when datasets are imbalanced.
\end{itemize}

The rest of the paper is organized as follows. We first discuss the related work in Section~\ref{sect:related-work}. The problem formulation of online AUC optimization is given in Section~\ref{sect:problem-formulation}. In Section~\ref{sect:proposed-algorithm}, we present the new algorithm~\textsc{FTRL-AUC} with time complexity and regret bound analysis. We evaluate our method in Section~\ref{sect:experiments}. Section~\ref{sect:conclusion} concludes the paper. For the reproducibility purpose, source code of \textsc{FTRL-AUC} including all baseline methods and datasets can be accessed at: \url{https://github.com/baojianzhou/ftrl-auc}.

\section{Related Work}
\label{sect:related-work}

\noindent{\bf Online AUC optimization.\quad}   AUC optimization algorithms have been developed under batch learning setting~\cite{joachims2005support,herschtal2004optimising,Xinhua} where the predictor is generated based on the entire training samples. The reason why designing online AUC maximization methods is challenging is twofold:  1) in contrast to accuracy-based  classification approaches where the loss is based on one individual example, the loss function of AUC optimization involves a pair of  examples;  2) in practice, individual examples are arriving sequentially rather than pairs of examples.

The first work of online AUC optimization has been proposed in~\cite{zhao2011online,kar2013generalization,Wang2012}. The key idea is to use the gradient of a local empirical error which compares the current training example with all previous ones. However, these methods need to save all (or part of) previous training samples and need to compare the current example with previous ones, which leads to high per-iteration $\mathcal{O}(t d)$ at time $t$.
The appealing work by Gao et. al~\cite{gao2013one} follows the same spirit but observes, in the case of the least square loss, that updates of such algorithms only rely on covariance matrix. This leads to  a per-iteration cost $\mathcal{O}(d^2).$  Similarly, Ding et. al~\cite{ding2015adaptive}  use the same formulation but with an adaptive learning rate with similar per-iteration cost.    

More recent studies~\cite{ying2016stochastic,liu2018fast,natole2018stochastic} develop online AUC optimization with per-iteration cost $\mathcal{O}(d)$ with competitive convergence results. The key idea is to introduce a new equivalent saddle-point formulation of AUC optimization. In particular,  it is shown that maximizing AUC score is equivalent to solving a min-max problem, and hence stochastic primal-dual gradient-based algorithms can be applied. The work \cite{liu2018fast,natole2018stochastic} develops fast algorithms for online AUC maximization with sparse regularization. 
However, all the methods do not take into account the inherent sparsity of data and the per-iteration cost is of $\mathcal{O}(d)$ which is still expensive if $d$ is very large. Moreover, such sparse online AUC optimization algorithms are essentially variants of stochastic proximal gradient algorithms which, as shown in \cite{xiao2010dual,langford2009sparse}, do not produce desired sparse solutions. 

\smallskip 
\noindent{\bf Online learning algorithms with sparsity.\quad} 
A natural approach to obtain sparsity under batch learning setting is to add the $\ell^1$ regularization. However, this is not the case for online setting where only one training sample is available at each time. Indeed, this is shown in \cite{langford2009sparse} that simply applying $\ell^1$ regularization fails to work since the gradient of each sample does not induce sparsity. Many important works~\citep{duchi2009efficient,langford2009sparse,xiao2010dual,duchi2011adaptive,yang2010online} have been proposed to capture problem sparsity. For example, the regularized dual-averaging method proposed in~\cite{xiao2010dual} captures sparsity more effectively. The idea of online dual averaging is based on~\cite{nesterov2009primal} as all gradients used are equally important.

 As demonstrated in these works \cite{mcmahan2010adaptive,mcmahan2011follow,mcmahan2013ad}, the regularized dual averaging can be regarded as a special case of the follow-the-regularized-leader. This type of algorithms is originally proposed in~\cite{Shalev-Shwartz2007} which is based on the follow-the-leader~\cite{hannan1957approximation}. It has been shown that both Regularized Dual Averaging (\textsc{RDA})~\cite{xiao2010dual} and Follow The Proximally Regularized Leader (\textsc{FTPRL})~\cite{mcmahan2010adaptive} can capture the sparsity effectively. However, the FTRL and regularized dual averaging methods are developed for accuracy-based loss depending on individual examples,  while the objective function of AUC maximization is of U-Statistics based on pairs of examples. It remains unclear that how we can incorporate the FTRL framework with the AUC setting to design efficient online AUC optimization methods exploring inherent sparsity of high-dimensional sparse data.   

We aim to design online AUC maximization algorithms which not only achieve much cheaper $O(k)$ per-iteration cost but also generate much sparser solutions than existing methods. By this, we mean the following two requirements.  The first one is that the gradient of loss of each training sample should be sparse. The other one is that the loss needs to be convex with respect to current model variable in order to guarantee a sublinear or logarithmic regret bound.  We will illustrate our new developments in the subsections to  achieve this goal.

\section{Problem Formulation}
\label{sect:problem-formulation}

In this section, we give the definition of AUC score and then define our problem. To this end, let us introduce some notations. 

The training sample at time $t$ is denoted as $\bm x_t \in \mathbb{R}^d$ and the corresponding label is $y_t \in \{+1,-1\}$. The model of a specific algorithm at time $t$ is denoted as $\bm w_t \in \mathbb{R}^d$. Define $\bm x_t^+ :=\bm x_t$ if $y_t = 1$ and $\bm x_t^- :=\bm x_t$ if $y_t = -1$. The $i$-th entries of $\bm x_t$ is denoted by $x_{t,i}$. The indicator function is defined as $\mathbb{1}_{[A]}$ where it takes value 1 if $A$ is true and 0 otherwise. We assume each training sample is $k$-sparse, i.e. $\|\bm x_t\|_0 = \mathcal{O}(k)$ and $k \ll d$.\footnote{$\|\bm x\|_0 := |\{x_i:x_i \ne 0\}|$, the number of non-zeros in $\bm x$. $k\ll d$ is not a hard condition. As long as $k < d$ on average over all training samples, our proposed method can benefit from it accordingly.} Matrices are denoted by bold capitals such as $\bm Q$. The average of all $t$ positive and $t$ negative data samples are denoted respectively by  $\overline{\bm x}_t^+=\frac{\sum_{i=1}^t \bm x_i^+}{t}$  and $\overline{\bm x}_t^-=\frac{\sum_{j=1}^t \bm x_j^-}{t}$.

\subsection{Definition of AUC Optimization}
Given a set of training samples $\mathcal{D} :=\{\bm x_i,y_i\}_{i=1}^T$ where $T=T_+ + T_-$, $T_+$ is the number of positive samples and $T_-$ is the number of negative samples in $\mathcal{D}$, the AUC score~\cite{hanley1982meaning} of a specific linear classifier $\bm w$ is defined as
\begin{small}
\begin{equation}
   \operatorname{AUC}(\bm w) := \frac{1}{T_+ T_-} \sum_{i=1}^{T} \sum_{j=1}^{T} \mathbb{1}_{\left[\bm w^\top \bm x_i \geq \bm w^\top \bm x_j\right]} \mathbb{1}_{\left[y_i = 1\right]} \mathbb{1}_{\left[y_j = -1\right]}. \label{inequ:auc-score}
\end{equation}
\end{small}

\begin{figure}
\centering
\includegraphics[width=6.81cm,height=3.46cm]{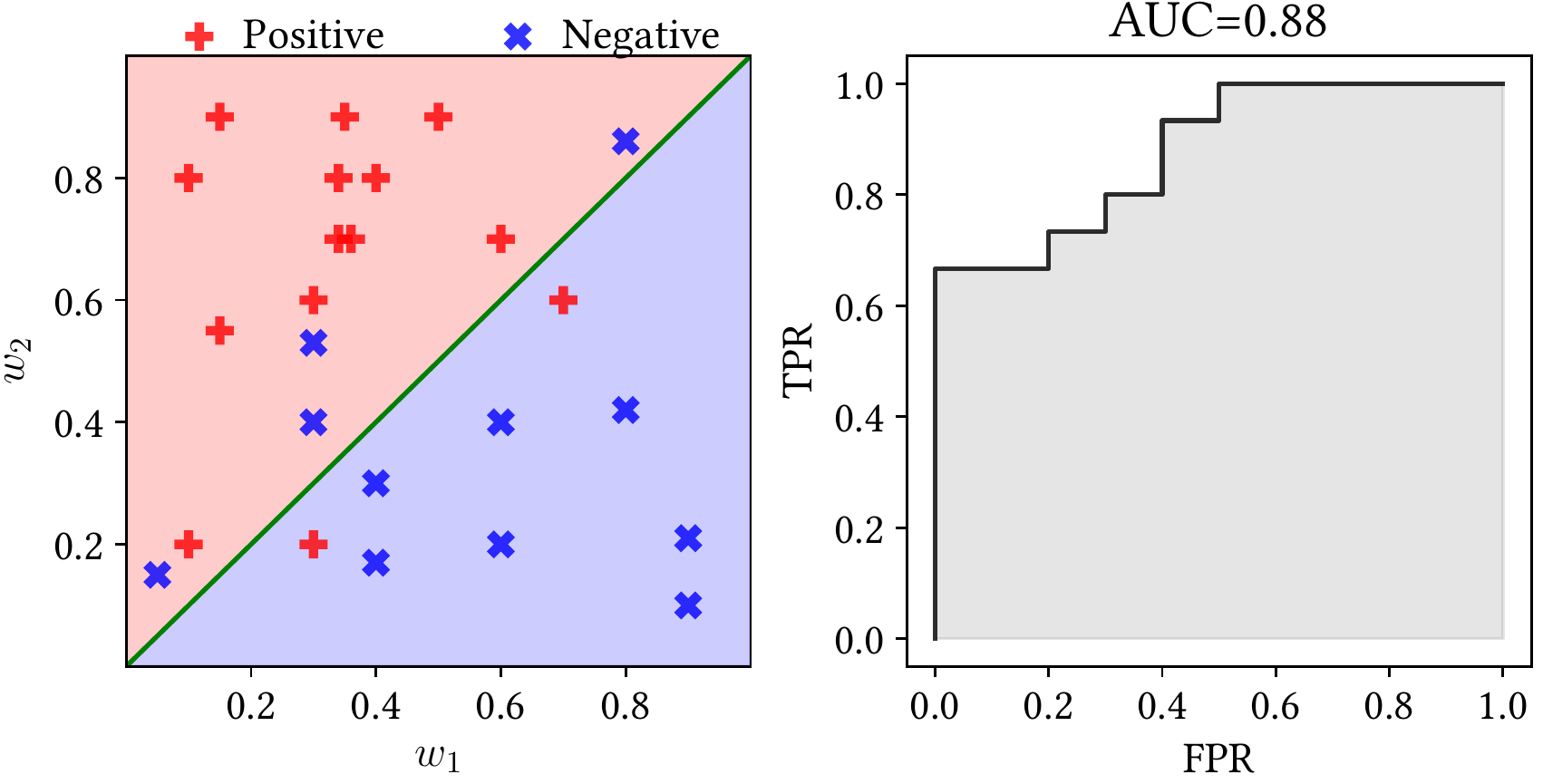}
\caption{A toy example of calculating AUC score}
\vspace{-5mm}
\label{fig:toy-example}
\end{figure}

Intuitively, $\operatorname{AUC}(\bm w)$ defined in~(\ref{inequ:auc-score}) is equivalent to the probability of a positive sample ranked by $\bm w$ higher than negative samples. As illustrated on the left of Figure~\ref{fig:toy-example}, the diagonal line is the linear classifier $\bm w = [-1,1]^\top$, which scores all points in red region positive and all points in blue region negative. By using the true labels of these data samples and these corresponding scores, we can draw the Area Under the ROC curve accordingly (as shown on the right of Figure~\ref{fig:toy-example}), where FPR and TPR are the false positive rate and true positive rate respectively.

To simplify~(\ref{inequ:auc-score}), denote positive samples as $\bm x_1^+,\bm x_2^+,\ldots, \bm x_{T_+}^+$ and negative samples as $\bm x_1^-,\bm x_2^-,\ldots,\bm x_{T_-}^-$ respectively. To maximize~(\ref{inequ:auc-score}), it is equivalent to minimize $1-AUC(\bm w)$, and hence we have the following minimization problem
\begin{equation}
\min_{\bm w\in \mathbb{R}^d} \frac{1}{T_+ T_-} \sum_{i=1}^{T_+} \sum_{j=1}^{T_-} \mathbb{1}_{[\bm w^\top \bm x_i^+ < \bm w^\top \bm x_j^- ]}.\label{inequ:objective-2}
\end{equation}
The inherently combinatorial of indicator function $\mathbb{1}_{[A]}$ makes it difficult to optimize~(\ref{inequ:objective-2}) directly. Thus, we replace the indicator function by a convex loss such as least square~\cite{gao2013one} or hinge loss~\cite{zhao2011online}. We choose the least square loss\footnote{It has been proved in~\cite{gao2013one,gao2015consistency}, least square is consistent with AUC but hinge loss not.} as our surrogate as the following
\begin{equation}
\min_{\bm w\in \mathbb{R}^d} F(\bm w) := \frac{1}{T_+ T_-} \sum_{i=1}^{T_+} \sum_{j=1}^{T_-} \left(1- \bm w^\top \left(\bm x_i^+  - \bm x_j^-\right) \right)^2, \label{inequ:auc-objective}
\end{equation}
where the objective function $F(\bm w)$ is in the form of U-Statistics \cite{cesa2006prediction} depending on pairs of individual examples. 

\subsection{Online AUC optimization}
The standard online learning algorithm is as follows: at each iteration $t$, the learner receives a question $\bm x_t$ and then the learner makes a prediction and a corresponding loss could occur after the true label $y_t$ is available, i.e. $f_t(\bm w_t; \{\bm x_t, y_t\})$. For example, a least square loss can be written as $f_t(\bm w_t; \{\bm x_t, y_t\}):=\left(\bm w_t^\top \bm x_t - y_t\right)^2$. The goal of the online learning is to design an algorithm which aims to minimize the regret as defined in the following
\begin{small}
\begin{equation}
\operatorname{Regret}_{T} := \sum_{t=1}^T f_t(\bm w_t; \{\bm x_t,y_t\}) - \min_{\bm w \in \mathbb{R}^d} \sum_{t=1}^T f_t(\bm w; \{\bm x_t,y_t\})
\label{inequ:regret-t}
\end{equation}
\end{small}
The main difficulty of applying online learning algorithms is that $F(\bm w)$ cannot be seamlessly decomposed into $T$ separable loss functions. Some theoretical results show that when the loss is pairwise loss functions~\cite{Wang2012,kar2013generalization}, it is still possible to do online learning, but existing strategies need to save previous training samples, which could be memory costly and the per-iteration cost is very high. 

To design an efficient online AUC optimization algorithm for high-dimensional sparse data with per iteration cost $\mathcal{O}(k)$, we reformulate the problem \eqref{inequ:auc-objective} as an empirical saddle point (min-max) formulation which is inspired by \cite{ying2016online}.  The proof in \cite{ying2016online} uses the concepts of conditional expectations and we provide a much simpler and straightforward proof.  

\begin{lemma}[Empirical Saddle Point Reformulation]
Minimization problem (\ref{inequ:auc-objective}) is equivalent to the following saddle point problem
\vspace{-5mm}
\begin{equation}
\min_{\substack{\bm w \in \mathbb{R}^d, \\ (a,b) \in \mathbb{R}^{2}}} \max_{\alpha \in \mathbb{R}} \Big\{ \sum_{t=1}^T f_t(\bm w,a,b,\alpha;\{\bm x_t,y_t\})\Big\}, \label{inequ:obj:nips16}
\end{equation}
where each $f_t$ is defined as the following
\vspace{-2mm}
\begin{small}
\begin{align}
f_t&(\bm w,a,b,\alpha;\{\bm x_t, y_t\}) = \left(1-p_T\right) \left(\bm w^\top \bm x_t - a\right)^2 \mathbb{1}_{[y_t = 1]} \nonumber \\
&\quad + p_T \left(\bm w^\top \bm x_t - b \right)^2 \mathbb{1}_{[y_t = -1]} - p_T (1-p_T)\alpha^2 \nonumber\\
&\quad + 2 (1+\alpha)\bm w^\top \bm x_t \Big(p_T\mathbb{1}_{[y_t = -1]} - (1-p_T)\mathbb{1}_{[y_t =1]} \Big),
\label{inequ:each-fi}
\end{align}
\end{small}
where $p_T = \sum_{i=1}^T\mathbb{1}_{[y_i=1]}/T$ estimates the distribution of positive training samples and $a,b,\alpha \in \mathbb{R}$. 
\label{lemma:1}
\end{lemma}
\begin{proof}
The original minimization objective is the following
\begin{scriptsize}
\begin{align*}
F(\bm w) &= \frac{1}{T_+ T_-} \sum_{i=1}^{T_+} \sum_{j=1}^{T_-} \left(1- \bm w^\top \left(\bm x_i^+  - \bm x_j^-\right) \right)^2 \\
&= 1 - \frac{2}{T_+}\sum_{i=1}^{T_+}\bm w^\top \bm x_i^+ + \frac{2}{T_-}\sum_{j=1}^{T_-}\bm w^\top \bm x_j^- + \frac{1}{T_+} \sum_{i=1}^{T_+}\left(\bm w^\top \bm x_i^+\right)^2 \\
&\quad + \frac{1}{T_-} \sum_{j=1}^{T_-}\left(\bm w^\top \bm x_j^-\right)^2 - \frac{2}{T_+ T_-} \sum_{i=1}^{T_+}\sum_{j=1}^{T_-}\left(\bm w^\top \bm x_i^+\right)\left(\bm w^\top \bm x_j^-\right) \\
&= 1 - \frac{2}{T_+}\sum_{i=1}^{T_+}\bm w^\top \bm x_i^+ + \frac{2}{T_-}\sum_{j=1}^{T_-}\bm w^\top \bm x_j^- \\
&\quad + \underbrace{\left\{\frac{1}{T_+} \sum_{i=1}^{T_+}\left(\bm w^\top \bm x_i^+\right)^2 - \left(\frac{1}{T_+} \sum_{i=1}^{T_+}\bm w^\top \bm x_i^+\right)^2 \right\}}_{\bm A} \\
&\quad + \underbrace{ \left\{\frac{1}{T_-} \sum_{j=1}^{T_-}\left(\bm w^\top \bm x_j^-\right)^2 - \left(\frac{1}{T_-} \sum_{j=1}^{T_-}\bm w^\top \bm x_j^-\right)^2 \right\}}_{\bm B} \\
&\quad + \underbrace{\left(\frac{1}{T_+} \sum_{i=1}^{T_+}\bm w^\top \bm x_i^+ - \frac{1}{T_-}\sum_{j=1}^{T_-} \bm w^\top \bm x_j^- \right)^2}_{\bm C}.
\vspace{-8mm}
\end{align*}
\end{scriptsize}
\vspace{-2mm}
For items $\bm A, \bm B, $ and $\bm C$, we can reformulate them as
\begin{scriptsize}
\begin{align*}
\bm A &= \min_{a \in \mathbb{R}} \frac{1}{T_+} \sum_{i=1}^{T_+} \left(\bm w^\top \bm x_i^+ - a \right) ^2, \bm B = \min_{b \in \mathbb{R}} \frac{1}{T_-} \sum_{j=1}^{T_-} \left(\bm w^\top \bm x_j^- - b \right) ^2, \\
\bm C &= \max_{\alpha \in \mathbb{R}} \left\{ 2\alpha\left(\frac{1}{T_-}\sum_{j=1}^{T_-}\bm w^\top \bm x_j^- - \frac{1}{T_+}\sum_{i=1}^{T_+}\bm w^\top \bm x_i^+\right) - \alpha^2\right\}.
\end{align*}
\end{scriptsize}
Hence, we can rewrite $F(\bm w)$ as the following
\begin{scriptsize}
\begin{align*}
\begin{split}
F(\bm w) &= 1 - \frac{2}{T_+}\sum_{i=1}^{T_+}\bm w^\top \bm x_i^+ + \frac{2}{T_-}\sum_{j=1}^{T_-}\bm w^\top \bm x_j^- \\
&\quad+\min_{(a,b) \in \mathbb{R}^2} \left\{ \frac{1}{T_+} \sum_{i=1}^{T_+} \left(\bm w^\top \bm x_i^+ - a \right) ^2 + \frac{1}{T_-} \sum_{j=1}^{T_-} \left(\bm w^\top \bm x_j^- - b \right) ^2 \right\}\\
&\quad+\max_{\alpha \in \mathbb{R}} \left\{ 2\alpha\left(\frac{1}{T_-}\sum_{j=1}^{T_-}\bm w^\top \bm x_j^- - \frac{1}{T_+}\sum_{i=1}^{T_+}\bm w^\top \bm x_i^+\right) - \alpha^2\right\}.
\end{split}
\end{align*}
\end{scriptsize}
Finally, the minimization of $F(\bm w)$ can be expressed as the sum of separable loss functions depending on each single training sample pair as the following
\begin{small}
\begin{align*}
&\min_{\bm w} F(\bm w) = \min_{\substack{\bm w \in \mathbb{R}^d, \\(a,b)\in \mathbb{R}^{2}}} \max_{\alpha \in \mathbb{R}} \Bigg\{ 1 - \frac{2}{T_+}\sum_{i=1}^{T_+}\bm w^\top \bm x_i^+ + \frac{2}{T_-}\sum_{j=1}^{T_-}\bm w^\top \bm x_j^- \\
&\quad+  \frac{1}{T_+} \sum_{i=1}^{T_+} \left(\bm w^\top \bm x_i^+ - a \right) ^2 + \frac{1}{T_-} \sum_{j=1}^{T_-} \left(\bm w^\top \bm x_j^- - b \right) ^2 \\
&\quad+  2\alpha\left(\frac{1}{T_-}\sum_{j=1}^{T_-}\bm w^\top \bm x_j^- - \frac{1}{T_+}\sum_{i=1}^{T_+}\bm w^\top \bm x_i^+\right) - \alpha^2 \Bigg\} \\
&= \min_{\substack{\bm w \in \mathbb{R}^d, \\(a,b)\in \mathbb{R}^{2}}} \max_{\alpha \in \mathbb{R}} \sum_{t=1}^T \Bigg\{ \frac{1}{T} - \frac{2}{T_+}\bm w^\top \bm x_t \mathbb{1}_{[y_t=1]} + \frac{2}{T_-} \bm w^\top \bm x_t\mathbb{1}_{[y_t=-1]} \\
&\quad+  \frac{1}{T_+}\left(\bm w^\top \bm x_t - a \right)^2\mathbb{1}_{[y_t=1]} + \frac{1}{T_-}\left(\bm w^\top \bm x_t - b \right) ^2\mathbb{1}_{[y_t=-1]} \\
&\quad+  2\alpha\left(\frac{1}{T_-}\bm w^\top \bm x_t\mathbb{1}_{[y_t=-1]} - \frac{1}{T_+} \bm w^\top \bm x_t\mathbb{1}_{[y_t=1]}\right) - \frac{\alpha^2}{T} \Bigg\} 
\end{align*}
\end{small}
Normalize the right hand side above by $\frac{T_+ T_-}{T}$ and notice that $p_T = T_+ /T$ and $1-p_T = T_-/T$, we finally arrive at the following equivalent optimization problem
\begin{equation}
\min_{\substack{\bm w \in \mathbb{R}^d,\\(a,b) \in \mathbb{R}^2}} \max_{\alpha \in \mathbb{R}} \sum_{t=1}^T \Big\{ f_t(\bm w,a,b,\alpha; x_t,y_t) \Big\}. \nonumber
\end{equation}
\end{proof}

Lemma~\ref{lemma:1} decomposes the original AUC objective function into the sum of separable loss by introducing three auxiliary variables $(a, b, \alpha)$.  However, the algorithm  in~\cite{ying2016stochastic} needs to update $\bm w, a$, and $b$ by using gradient descent and to update $\alpha$ by using gradient ascent. More importantly, the $\bm w_t$ at iteration $t$ needs to project back to a $\ell^2$-norm ball, which requires per-iteration cost  $\mathcal{O}(d)$. A recent work~\cite{natole2018stochastic} shows that, it is not necessary to update $a, b, $ and $\alpha$ and optimization problem~(\ref{inequ:obj:nips16}) can be further reformulated as 
\begin{equation}
\min_{\substack{\bm w \in \mathbb{R}^d}} \Big\{ \sum_{t=1}^T f_t(\bm w,a(\bm w),b(\bm w),\alpha(\bm w);\{\bm x_t,y_t\})\Big\}, \label{inequ:obj:icml18}
\end{equation}
where $a(\bm w) = \bm w^\top \overline{\bm x}_{T_+}^+$, $b(\bm w) = \bm w^\top \overline{\bm x}_{T_-}^-$, and $\alpha(\bm w)=b(\bm w) - a(\bm w)$. From now on, we denote $f_t(\bm w,a(\bm w),b(\bm w),\alpha(\bm w);\{\bm x_t,y_t\})$ as $f_t(\bm w)$. The gradient of each $f_t(\bm w)$ above is
\begin{small}
\begin{align*}
\begin{split}
&\nabla f_t(\bm w) = 2(1-p_T)\left(\bm w^\top \bm x_t - \bm w^\top \overline{\bm x}_{T_+}^+\right)\left(\bm x_t - \overline{\bm x}_{T_+}^+\right)\mathbb{1}_{[y=1]} \\
&\quad + 2p_T(\bm w^\top \bm x_t - \bm w^\top \overline{\bm x}_{T_-}^-)(\bm x_t - \overline{\bm x}_{T_-}^-)\mathbb{1}_{[y=-1]} \\
&\quad - 2p_T(1-p_T)\left(\bm w^\top\overline{\bm x}_{T_-}^- - \bm w^\top\overline{\bm x}_{T_+}^+ \right)\left(\overline{\bm x}_{T_-}^- - \overline{\bm x}_{T_+}^+\right) \\
&\quad + 2(1+\bm w^\top \overline{\bm x}_{T_-}^- - \bm w^\top\overline{\bm x}_{T_+}^+)(p_T\mathbb{1}_{[y_t = -1]} - (1-p_T)\mathbb{1}_{[y_t =1]})\bm x_t \\
&\quad + 2\bm w^\top \bm x_t (p_T\mathbb{1}_{[y_t = -1]} - (1-p_T)\mathbb{1}_{[y_t =1]})( \overline{\bm x}_{T_-}^- - \overline{\bm x}_{T_+}^+).
\end{split}
\end{align*}
\end{small}

However, there are two disadvantages to apply~(\ref{inequ:obj:icml18}) to the online learning setting: 1) the estimators $p_T,\overline{\bm x}_{T_+}^+$ and $\overline{\bm x}_{T_-}^-$ need to be known, which is unrealistic in the online learning setting; 2) $f_t(\bm w)$ is not convex as proved in~\cite{ying2019spauc} and the gradient of $f_t(\bm w)$ with respect to $\bm w$ is not sparse in general. The gradient is non-sparse because it needs to calculate both $\bm w^\top \overline{\bm x}_{T_+}^+$ and $\bm w^\top \overline{\bm x}_{T_-}^-$. These vectors are not sparse and need $\mathcal{O}(d)$ cost in general. In the subsequent section, we will introduce several novel techniques to resolve these issues. 

\section{Proposed Algorithm: \textsc{FTRL-AUC}}
\label{sect:proposed-algorithm}

\subsection{Problem reformulation}
\label{sect:4:sparse-gradient}
The difficultly of getting a sparse gradient motivates us to find an alternative way. Here is the main idea to overcome this obstacle.   At each iteration $t$, instead of using   $f_t(\bm w_t)$, we design a new convex loss $\widehat{f}_t(\bm w_t)$ by replacing the term $a(\bm w_t) = \bm w_t^T \overline{\bm x}_{T_+}^+$ by $a_t(\bm w_{t-1}) = \bm w_{t-1}^T \overline{\bm x}_{t_+}^+$, $b(\bm w_t) = \bm w_t^T \overline{\bm x}_{T_-}^-$ by $b_t(\bm w_{t-1}) = \bm w_{t-1}^T \overline{\bm x}_{t_-}^-$, and $p_T$ by $p_t$, where $\overline{\bm x}_{t_+}^+$ $ \overline{\bm x}_{t_-}^- $, and $p_t$ are the current available estimators at time $t$. The precise form of this new function  $\widehat{f}_t(\bm w_t)$ is described in the following theorem. 

\begin{theorem}
Define the loss of sample $\{\bm x_t,y_t\}$ as the following
\begin{small}
\begin{align}
\widehat{f}_t &(\bm w_t;\{\bm x_t,y_t\}) := \left(1-p_t\right) \left(\bm w_t^\top \bm x_t - a_t(\bm w_{t-1})\right)^2 \mathbb{1}_{[y_t =1]} \nonumber \\
&+ p_t \left(\bm w_t^\top \bm x_t - b_t(\bm w_{t-1}) \right)^2 \mathbb{1}_{[y_t = -1]} - p_t (1-p_t)(\alpha_t(\bm w_{t-1}))^2 \nonumber\\
&+ 2 (1+\alpha_t(\bm w_{t-1}))\bm w_t^\top \bm x_t \Big(p_t\mathbb{1}_{[y_t = -1]} - (1-p_t)\mathbb{1}_{[y_t =1]} \Big), \label{inequ:new-ft}
\end{align}
\end{small}
where $p_t = \frac{\sum_{i=1}^t \mathbb{1}_{[y_i = 1]}}{t}, \alpha_t(\bm w_{t-1}) = b_t(\bm w_{t-1}) - a_t(\bm w_{t-1})$
\begin{small}
\begin{align*}
a_t(\bm w_{t-1}) &= \bm w_{t-1}^\top \frac{\sum_{i=1}^t \bm x_i\mathbb{1}_{[y_i = 1]}}{t}, \text{ and }\\ 
b_t(\bm w_{t-1}) &= \bm w_{t-1}^\top\frac{\sum_{i=1}^t \bm x_i\mathbb{1}_{[y_i = -1]}}{t}.
\end{align*}
\end{small}
We denote the above definition as $\widehat f_t(\bm w_t)$ and it is convex with respect to $\bm w_t$. If $\|\bm x_t\|_0 = \mathcal{O}(k)$, then gradient of $\widehat{f}_t$ with respect to $\bm w_t$ is also a $k$ sparse vector, i.e., $\|\nabla \widehat{f}_t(\bm w_t)\|_0 = \mathcal{O}(k)$.
\label{thm:1}
\end{theorem}
\begin{proof}
Clearly, the convexity of $\widehat{f}_t(\bm w_t)$ can be checked by the Hessian matrix of $\widehat{f}_t(\bm w_t)$ with respect to $\bm w_t$, i.e,
\begin{equation}
H_{\widehat{f}_t}(\bm w_t) : = \left(2(1-p_t)\mathbb{1}_{[y_t=1]} + 2p_t\mathbb{1}_{[y_t=-1]}\right)\bm x_t\bm x_t^\top. \nonumber
\end{equation}
Notice that the Hessian matrix $H_{\widehat{f}_t}(\bm w_t)$ is a rank-one matrix with a nonnegative coefficient, and hence it is a positive semi-definite matrix. Therefore, $\widehat{f}_t$ is convex. The gradient of $\widehat{f}_t(\bm w_t)$ with respect to $\bm w_t$ can be calculated as the following
\begin{equation}
\nabla \widehat{f}_t(\bm w_t) = 
\begin{cases}
2(1-p_t)\left(\bm w_t^\top \bm x_t - b_t(\bm w_{t-1})-1\right)\bm x_t & y_t = 1 \\
2p_t\left(\bm w_t^\top \bm x_t - a_t(\bm w_{t-1})+1\right)\bm x_t & y_t = -1 
\end{cases}
\label{inequ:sparse-gradient}
\end{equation}
Notice that $\nabla \widehat{f}_t(\bm w_t)$ is a scaling of $\bm x_t$, hence, if $\|\bm x_t\|_0 = \mathcal{O}(k)$, then $\nabla \widehat{f}(\bm w_t)$ is also $k$ sparse.
\vspace{-3mm}
\end{proof}

Although the gradient of $\widehat{f}_t(\bm w_t)$ is sparse, it still needs to have $\mathcal{O}(d)$ cost by observing that~(\ref{inequ:sparse-gradient}) needs to calculate $a_t(\bm w_{t-1})$ and $b_t({\bm w}_{t-1})$. In Section~\ref{sect:4:sparse}, we show that it is possible to obtain a sparse gradient of $\widehat{f}_t(\bm w_t)$ by using a ``lazy update'' trick in $\mathcal{O}(k)$ time.

\subsection{Pursuing \texorpdfstring{$\mathcal{O}(k) $ }{Ok} per-iteration cost}
\label{sect:4:sparse}
Theorem~\ref{thm:1} demonstrates that it is possible to calculate the gradient in $\mathcal{O}(k)$ per-iteration. However, the major difficulty is that we need to estimate the expectation of empirical score for positive and negative sample with respect to $\bm w$ up to $t$ as the following
\begin{small}
\begin{equation}
a_t(\bm w) = \bm w^\top \frac{\sum_{i=1}^t \bm x_i\mathbb{1}_{[y_i = 1]}}{t},  b_t(\bm w) = \bm w^\top \frac{\sum_{i=1}^t \bm x_i\mathbb{1}_{[y_i = -1]}}{t}. \nonumber
\end{equation}
\end{small}
Given any $\bm w_{t-1}$, we can rewrite the above equations $a_t(\bm w_{t-1})$ and $b_t(\bm w_{t-1})$ as the following
\begin{equation}
a_t(\bm w_{t-1}) = \bm w_{t-1}^\top \frac{\sum_{i=1}^{t_+} \bm x_i^+}{t_+},\quad b_t(\bm w_{t-1}) = \bm w_{t-1}^\top \frac{\sum_{j=1}^{t_-} \bm x_j^-}{t_-}, \nonumber
\end{equation}
where $t = t_+ + t_-$. Take $a_t(\bm w_{t-1})$ as an example. We need to obtain the estimation of $a_t(\bm w_{t-1})$ where we denote it as $\widehat{a}_t(\bm w_{t-1})$. We reformulate the above equation and rewrite it as the following
\begin{align*}
&\widehat{a}_t(\bm w_{t-1}) := \bm w_{t-1}^\top \frac{\sum_{i=1}^{t_+-1} \bm x_i+\bm x_{t_+}^+}{t_+} \\
&= \frac{t_+ - 1}{t_+}\bm w_{t-1}^\top \frac{1}{t_+ - 1}\sum_{i=1}^{t_+-1} \bm x_i^+ +\frac{\bm w_{t-1}^\top \bm x_{t_+}^+}{t_+} \\
&\bm{\approx} \frac{t_+ - 1}{t_+}\bm w_{t-2}^\top \frac{1}{t_+ - 1}\sum_{i=1}^{t_+-1} \bm x_i^+ +\frac{\bm w_{t-1}^\top \bm x_{t_+}^+}{t_+}\ \text{(\textbf{Lazy update})} \\  &= \frac{t_+ - 1}{t_+} \widehat{a}_{t-1}(\bm w_{t-2}) +\frac{\bm w_{t-1}^\top \bm x_{t_+}^+}{t_+}.  
\end{align*}
Similarly, $\widehat{b}_t(\bm w_{t-1})$ can be done in the same way above. We propose to update $a_t(\bm w_{t-1}), b_t(\bm w_{t-1})$ by their estimations $\widehat{a}_t(\bm w_{t-1})$ and $\widehat{b}_t(\bm w_{t-1})$ as the following
\vspace{-2mm}
\begin{align*}
 \widehat{a}_t(\bm w_{t-1}) &= \frac{t_+ - 1}{t_+} \widehat{a}_{t-1}(\bm w_{t-2}) +\frac{\bm w_{t-1}^\top \bm x_{t_+}^+}{t_+},\\
 \widehat{b}_t(\bm w_{t-1}) &= \frac{t_- - 1}{t_-} \widehat{b}_{t-1}(\bm w_{t-2}) +\frac{\bm w_{t-1}^\top \bm x_{t_-}^-}{t_-}.
\end{align*}
Clearly, per-iteration costs of $\widehat{a}_t(\bm w_{t-1})$ and $\widehat{b}_t(\bm w_{t-1})$ are $\mathcal{O}(k)$. 

\begin{figure}
\centering
\includegraphics[width=8cm,height=3.33cm]{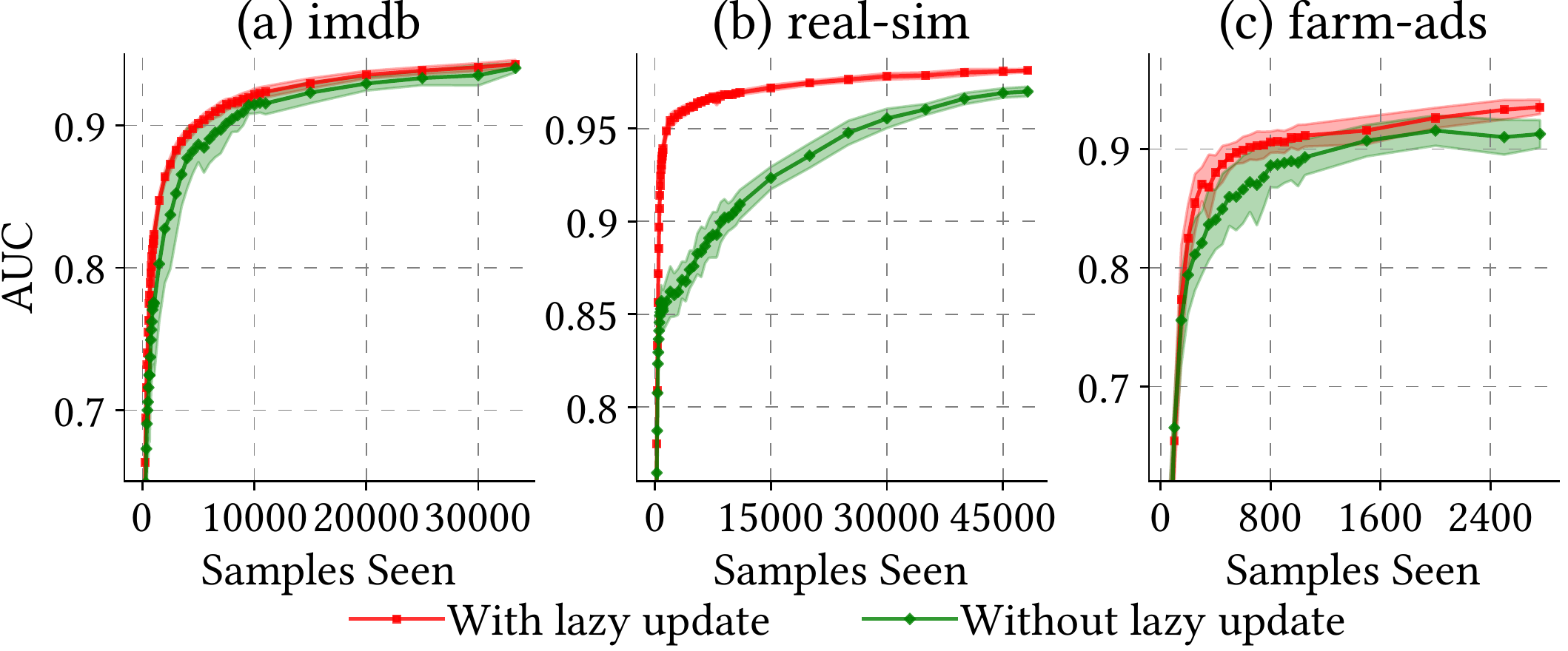}
\caption{Comparison between with and without lazy update}
\label{fig:compare-update-rule}
\end{figure}

One may argue, algorithms without lazy update, as usually expected and taken-for-granted, will outperform those with lazy update. However, in our proposed algorithm, things are just the opposite. As a preliminary study, we apply both the gradient with and without lazy update step to our algorithm and test the two methods on three collected real-world datasets, imdb, real-sim, and farm-ads. As shown in Figure~\ref{fig:compare-update-rule}, the results by using the lazy update rule are surprisingly better and more stable than those without using it. It means that algorithm with lazy update rule could converge faster than the one without lazy update rule, especially at the early stage of the online learning processing. This could be due to the fact that, although $\bm w_t$ does not well converge to the true model $\bm w^*$ at the early stage, the estimator $\widehat{a}_t(\bm w_{t-1})$ and $\widehat{b}_t(\bm w_{t-1})$ which take the average of all previous models, i.e., $\sum_{i=1}^t \bm w_i/t$, will lead to a more stable behavior.

\subsection{Sparsity-pursuing model}

In this sub-section, we assume the gradient $\nabla \widehat{f}_i(\bm w_i)$ is the gradient by replacing $a_t(\bm w_{t-1})$ and $b_t(\bm w_{t-1})$ with $\widehat{a}_t(\bm w_{t-1})$ and $\widehat{b}_t(\bm w_{t-1})$ respectively. To make our model sparse, we apply the generalized Follow-The-Regularized-Leader framework proposed in a series of works~\cite{mcmahan2010adaptive,mcmahan2011follow,mc2017}. For each time $t$, we propose to update the model $\bm w_t$ as 
\begin{align}
\bm w_{t+1} &= \argmin_{\bm w} \Big\langle \sum_{i=1}^t\nabla \widehat{f}_i(\bm w_i), \bm w \Big\rangle + t \lambda \| \bm w \|_1 \nonumber\\
&\quad + \frac{1}{2} \sum_{i = 1}^t \left\|\bm Q_t^{1/2} \left(\bm w - \bm w_i \right)\right\|_2^2, \label{update}
\end{align}
where $\bm Q_t:= \operatorname{diag}(q_{t,1},q_{t,2},\ldots,q_{t,d})$ is a positive definite matrix to control the adaptivity of regularizers and corresponding learning rate. This regularization term makes the updated model more stable~\cite{shalev2012online}.  We adopt the learning rate strategy suggested in~\cite{mcmahan2011follow} where $\eta_{t,i} = \frac{1}{\sum_{j=1}^t q_{j,i}}$. To pursue sparsity, we add $\ell^1$-norm with parameter $\lambda$. The above minimization problem has the following closed-form solution
\begin{equation}
\bm w_{t+1} = -\eta_t \sign(\bm z_t)[|\bm z_t| - \lambda]_+, \nonumber
\end{equation}
where all operations are element-wise. More specifically, the $i$-th coordinate, $\bm z_t$ can be calculated recursively as the following
\begin{equation}
\bm z_{t,i} = \bm z_{t-1,i} + \bm \nabla \widehat{f}_{t,i}(\bm w_{t,i}) + \left(\frac{1}{\eta_{t,i}} - \frac{1}{\eta_{t-1,i}}\right) \bm w_{t,i}, \nonumber
\end{equation}
where the per-coordinate learning rate is $\eta_{t,i} = \frac{\gamma}{1+\sqrt{\sum_{j=1}^t \nabla \widehat{f}_{j,i}(\bm w_{j,i})^2}}$. Again, according to the previous analysis in Section~\ref{sect:4:sparse-gradient} and ~\ref{sect:4:sparse}, the gradient is $k$ sparse and can be updated by using $\mathcal{O}(k)$ time per-iteration. Interestingly, by using the above learning rate schedule, we do not need to update $\bm w_{t+1,i}$ of the $i$-th coordinate whenever $\bm x_{t,i} = 0$. Notice that $\nabla \widehat{f}_{t,i}(\bm w_{t,i}) = 0$ if $\bm x_{t,i} = 0$ because the gradient is the scaling vector of $\bm x_t$. Furthermore, we do not need to update ${\bm z_{t,i}}$ either if $i$-th entry of $\bm x_t$ is 0 by noticing the following equation
\begin{small}
\begin{align*}
\frac{1}{\eta_{t,i}} - \frac{1}{\eta_{t-1,i}} & = \frac{\sqrt{\sum_{j=1}^t \nabla \widehat{f}_{j,i}(\bm w_{j,i})^2} - \sqrt{\sum_{j=1}^{t-1} \nabla \widehat{f}_{j,i}(\bm w_{j,i})^2}}{\gamma}\\
&= \frac{\sqrt{\sum_{j=1}^{t-1} \nabla \widehat{f}_{j,i}(\bm w_{j,i})^2 + \nabla \widehat{f}_{t,i}(\bm w_{t,i})^2}}{\gamma} \\
&\quad- \frac{\sqrt{\sum_{j=1}^{t-1} \nabla \widehat{f}_{j,i}(\bm w_{j,i})^2}}{\gamma} = 0,
\end{align*}
\end{small}
where the last equation due to $\nabla \widehat{f}_{t,i}(\bm w_{t,i}) = 0$. It means that $\bm z_t$ is also $k$-sparse, which could update $\bm w_{t+1}$ in $\mathcal{O}(k)$.

\begin{algorithm}[H]
\caption{\small Follow-The-Regularized-Leader AUC (FTRL-AUC)}
\begin{algorithmic}[1]
\STATE \textbf{Input}: $\gamma,\lambda$
\STATE {\small $t_+ = 0, t_- = 0, p_0 = 0, \bm w_0 = \bm 0, \widehat{a}_0(\bm w_{-1}) = 0, \widehat{b}_0(\bm w_{-1}) = 0$}
\STATE $\bm z = \bm 0, \bm v = \bm 0, \bm w_0 = \bm 0$
\FOR{$t= 0, 1, 2, \ldots$}
\STATE receive $\{{\bm x}_t,y_t\}$
\STATE predict the score $\bm w_t^\top \bm x_t$ 
\IF{$y_t = 1$}
\STATE ${\bm g}_t = 2(1-p_t)(\bm w_t^\top \bm x_t - \widehat{b}_t(\bm w_{t-1}) - 1 )\cdot\bm x_t$  
\STATE $p_{t+1} = \frac{t}{t+1}p_t + \frac{1}{t+1}$
\STATE $t_+ = t_+ + 1$
\STATE $\widehat{a}_{t+1}(\bm w_t) = \frac{t_+ - 1 }{t_+} \widehat{a}_t(\bm w_{t-1}) + \frac{\bm w_t^\top \bm x_t}{t_+}$ 
\ELSE
\STATE ${\bm g}_t = 2p_t(\bm w^\top \bm x_t - \widehat{a}_t(\bm w_{t-1}) + 1)\bm x_t$ 
\STATE $p_{t+1} = \frac{t }{t + 1}p_t$ 
\STATE $t_- = t_- + 1$
\STATE $\widehat{b}_{t+1}(\bm w_t) = \frac{t_- - 1 }{t_-} \widehat{b}_t(\bm w_{t-1}) + \frac{\bm w_t^\top \bm x_t}{t_-}$
\ENDIF
\FOR{$i = \{j: x_{t,j} \ne 0, j \in \{1,2,\ldots,d\}\}$}
\STATE $ w_{t+1,i} = -\frac{\gamma}{1+\sqrt{v_i}}\sign(z_i)\left[|z_i|-\lambda\right]_+$
\STATE $z_i = z_i + g_{t,i} - \frac{1}{\gamma}( \sqrt{v_i+g_{t,i}^2} - \sqrt{v_i}) w_{t,i}$
\STATE $v_i = v_i + g_{t,i}^2 $
\ENDFOR
\ENDFOR
\end{algorithmic}\label{alg:ftrl-auc}
\end{algorithm}

To summarize, the proposed algorithm \textsc{FTRL-AUC} is presented in Algorithm~\ref{alg:ftrl-auc}. It has two input parameters, the initial learning rate $\gamma$ and sparse regularization ($\ell^1$-norm) parameter $\lambda$. At the beginning, it initializes $p_0 = 0, \bm w_0 = \bm 0, \widehat{a}_0(\bm w_{-1}) = 0, \widehat{b}_0(\bm w_{-1}) = 0$. At each time $t$, executes the following two main steps to update $\bm w_t$ :
\begin{itemize}[leftmargin=*]
\item At each time $t$, it receives a sample $\bm x_t$ in Line 5. The gradient of $\widehat{f}_t(\bm w_t)$ is $\bm g_t:= \nabla \widehat{f}_t(\bm w_t)$ as shown in~(\ref{inequ:sparse-gradient}) which is calculated in Line 8 if $y_t$ is positive; otherwise in Line 13. We use the lazy update steps to update $\widehat{a}_t(\bm w_{t-1})$ and $\widehat{b}_t(\bm w_{t-1})$ in Line 11 and 16 respectively. The proportion of positive samples $p_t$ is updated in Line 9 and Line 14. Overall, it costs $\mathcal{O}(2k)$ from Line 5 to Line 17.
\item After taking the gradient of current training sample $\bm x_t$. The $\bm w_t$ will be updated accordingly in Line 19. All previous gradients are accumulated in vector $\bm z$ in Line 20. From Line 18 to Line 22, it only needs $\mathcal{O}(k)$.
\end{itemize}

\subsection{Complexity and Regret Discussion}

\textbf{Time and space complexity.\quad} The time complexity of the per-iteration as claimed is $\mathcal{O}(k)$. It needs $\mathcal{O}(k)$ operations to calculate the gradient from Line 7 to Line 18 and to update model $\bm w_t$. Other parameters such as $\bm v,\bm z$ and $\bm g_t$ also need $O(k)$ operations. Hence, the total is $\mathcal{O}(k)$. Furthermore, \textsc{FTRL-AUC} is also space efficient. During the learning process, it only keeps track of four vectors, $\bm w_t,\bm z, \bm v$, and $\bm g_t$. Hence, the memory requirement is $\mathcal{O}(d)$.

\noindent\textbf{Regret discussion.\quad} We close this section with a brief discussion on the possibility of deriving regret bounds of our algorithm. The regret for the standard Follow-The-Regularized-Leader approach can be found in \cite[Theorem 2]{mc2017}. However, the results there do not directly apply to our case due to two main reasons: 1) the boundedness of $\bm w_{t}$ and $\widehat{f}_t(\bm w_t)$ are not obvious the update \eqref{update} is unconstrained; and 2) $\widehat{a}_t(\bm w_{t-1})$ and $\widehat{b}_t(\bm w_{t-1})$ are two estimators in which there is a gap between the approximated and true ones. Yet, we may still be able to establish the sub-linear regret bounds for our algorithms using the following ideas. 
Firstly, the uniform boundedness of $\|\bm w\|_2 $ and $\|\nabla \widehat{f}_t(\bm w)\|$ can be proved by carefully analyzing the equation \eqref{update} as the regularization term $\|\bm w\|_1$ will enforce the boundeness of the parameter $\bm w.$ Secondly,  one can use concentration inequalities to show that the approximate estimators $\widehat{a}_t(\bm w_{t-1})$ and $\widehat{b}_t(\bm w_{t-1})$ differ from their true estimators ${a}_t(\bm w_{t-1})$ and ${b}_t(\bm w_{t-1})$ by a very small term $\mathcal{O}({1\over \sqrt{t}})$ as long as the iteration number $t$ is very large.   We leave the detailed regret analysis as an interesting future work.

\section{Experiments}
\label{sect:experiments}
To verify \textsc{FTRL-AUC} in experiments, we aim to answer the following three questions:
\begin{itemize}[leftmargin=*]
\item \textbf{Q1}: Compared with the state-of-the-art online AUC optimization methods, can \textsc{FTRL-AUC} significantly shorten the run time?
\item \textbf{Q2}: Can \textsc{FTRL-AUC} capture sparsity more effectively without significant loss of performance on AUC score?
\item \textbf{Q3}: Does \textsc{FTRL-AUC} have any advantages over the online learning method for logistic loss when the dataset is imbalanced?
\end{itemize}

\subsection{Datasets and baseline methods}

\textbf{Datasets.\quad} We consider the high-dimensional sparse datasets with the task of binary classification. All datasets are summarized in Table~\ref{table:datasets}. More specifically, $n$ is the total number of samples, $n_+$ and $n_-$ is the total number of positive and negative samples respectively. Datasets of real-sim, rcv1b, news20b, and avazu can be downloaded from~\cite{libsvm}. The dataset of farm-ads~\cite{MC2011} can be downloaded from~\cite{NHBM1998}. There are two sentiment classification datasets, reviews~\cite{acl2007} and imdb~\cite{MDP2011}. We also consider a click-through rate prediction dataset avazu~\cite{ffm2016} which has 1 million features and about 14 millions of training samples.

\begin{table}[ht]
\setlength{\tabcolsep}{4pt}
\centering
\caption{High-dimensional sparse datasets}
\label{table:datasets}
\begin{threeparttable}[t]
\begin{tabular}{l l l l l l}
\hline
\rowcolor{gray!40}
datasets & $n$ & $n_{+}$ & $n_{-}$ & $d$ & $\mathcal{O}(k)$\tnote{1}\\
\hline
real-sim & 72,309 & 22,238 & 50,071 & 20,958 & 52 \\
\hline
rcv1b & 697,641 & 365,951 & 331,690 & 46,674 & 74 \\
\hline
farm-ads & 4,143 & 2,210 & 1,933 & 54,876 & 198 \\
\hline
imdb & 50,000 & 25,000 & 25,000 & 89,527 & 136 \\
\hline
reviews & 8,000 & 4,000 & 4,000 & 473,856 & 190 \\
\hline
news20b & 19,996 & 9,999 & 9,997 & 1,355,191 & 455 \\
\hline
avazu & 14,596,137 & 1,734,407 & 12,861,730 & 1,000,000 & 15 \\
\hline
\end{tabular}
\begin{tablenotes}
\item[1] $\mathcal{O}(k)$ of each dataset is calculated by $\ceil*{\sum_{i=1}^n \|\bm x_i\|_0 / n}$. $\mathcal{O}(k)$ tells the number of nonzero entries on $\bm x_i$ on average.
\end{tablenotes}
\end{threeparttable}%
\end{table}

\noindent\textbf{Baseline methods.\quad} We consider two types of method. The first type is for online AUC optimization. It has six methods, including \textsc{SOLAM}, a stochastic online AUC Maximization method proposed in~\cite{ying2016stochastic}, \textsc{SPAM}, a 
stochastic proximal AUC maximization algorithm put forward in~\cite{natole2018stochastic} (\textsc{SPAM}-$\ell^1$, \textsc{SPAM}-$\ell^2$, and \textsc{SPAM}-$\ell^1/\ell^2$ based on the different regularization strategy), \textsc{FSAUC},  a fast
stochastic algorithm for true AUC maximization designed in~\cite{liu2018fast}, and SPAUC proposed in~\cite{ying2019spauc}.\footnote{There are other online AUC optimization methods such as OAM~\cite{zhao2011online}, OPAUC~\cite{gao2013one} and AdaOAM~\cite{ding2015adaptive}. However, these algorithms either need to have $\mathcal{O}(d\times d)$ memory or $\mathcal{O}({T\times d})$ run time, which makes them hard to be applied to high-dimensional sparse datasets.} The second type is the online learning methods that minimize the logistic loss. We mainly consider the sparse-pursuing ones, including \textsc{RDA}-$\ell^1$~\cite{xiao2010dual}, \textsc{AdaGrad}~\cite{duchi2011adaptive}, and \textsc{FTRL-Pro}~\cite{mcmahan2013ad} which is an essentially Follow-The-Regularized-Leader approach proposed in~\cite{mcmahan2010adaptive}. 

\noindent\textbf{Experimental setup.\quad} All methods are implemented in C language with a Python2.7 wrapper. We split all datasets in the following way: 4/6 samples are for training and the rest two 1/6 samples  for validation and testing respectively. We repeat this procedure 10 times and report the results over these 10 trials. Parameter tuning and other detailed experimental setup can be found in the appendix. All methods stop after seeing all training samples. That is, all methods pass training samples once. AUC scores of all convergence curves in our experiments are calculated by using testing datasets.

\subsection{Run time performance}
To answer \textbf{Q1}, as we have claimed, one of the main improvements of our method over baselines is the time complexity. We test all methods on the six high-dimensional datasets and the run time has been shown in the top section of Table~\ref{table:auc-performance}. Clearly, the run time of \textsc{FTRL-AUC} significantly outperforms all the other methods by a large margin. For example, in news20b dataset, it is about 887 times faster than the existing fastest baseline, i.e., \textsc{SPAM}-$\ell^1$ and about 1981 times faster than the slowest, i.e., \textsc{FSAUC}. The run time of \textsc{FSAUC} is worse than the others. This is because it needs to projection $w_t$ on $\ell^1$-ball and it is time consuming. However, our method only needs to have $\mathcal{O}(k)$ multiplication per-iteration, thus time-efficient. For example, in rcv1b dataset, \textsc{FTRL-AUC} only uses 1.609 seconds in average to process about 465,094 training samples with dimension $d=46,674$.

\renewcommand{\arraystretch}{1.1}
\begin{table*}[ht!]
\caption{The performance of AUC optimization methods with respect to run time (seconds), sparse ratio, and AUC score}
\centering
 \begin{threeparttable}[t]
\begin{tabular}{|P{0.07\textwidth}|P{0.11\textwidth}|P{0.095\textwidth}|P{0.095\textwidth}|P{0.095\textwidth}|P{0.11\textwidth}|P{0.11\textwidth}|P{0.095\textwidth}|}
\hline 
 Datasets\tnote{1} & FTRL-AUC & SPAM-$\displaystyle \ell^{1}$ & SPAM-$\displaystyle \ell^{2}$ & SPAM-$\displaystyle \ell^{1} /\ell^{2}$  & SOLAM & SPAUC & \textsc{FSAUC\tnote{2}} \\
 \hline 
\rowcolor{gray!40}
 \multicolumn{8}{|c|}{Run Time\tnote{3} \ (seconds)} \\
\hline 
farm-ads & \textbf{0.015$\pm$0.007} & 1.597$\pm$0.014 & 0.746$\pm$0.221 & 1.586$\pm$0.053 & 1.614$\pm$0.430 & 1.875$\pm$0.137 & 2.346$\pm$0.335  \\\hline
real-sim & \textbf{0.106$\pm$0.029} & 10.513$\pm$0.284 & 3.562$\pm$0.855 & 11.521$\pm$0.696 & 9.741$\pm$3.307 & 14.200$\pm$0.543 & 17.623$\pm$2.704 \\\hline 
rcv1b &  \textbf{1.609$\pm$0.155} & 230.769$\pm$7.121 & 103.251$\pm$25.46 & 234.379$\pm$0.491 & 178.989$\pm$31.705 & 334.820$\pm$28.554 & 424.636$\pm$73.22 \\\hline 
news20b  & \textbf{0.324$\pm$0.059} & 287.676$\pm$18.94 & 355.556$\pm$30.83 & 397.424$\pm$1.688 & 404.911$\pm$151.445 & 347.866$\pm$8.039 & 642.008$\pm$105.9 \\\hline 
reviews & \textbf{0.025$\pm$0.015} & 28.334$\pm$6.604 & 26.529$\pm$10.369 & 35.035$\pm$0.607 & 15.054$\pm$5.854 & 22.944$\pm$5.437 & 21.740$\pm$7.169\\\hline
imdb &  \textbf{0.355$\pm$0.046} & 30.717$\pm$0.984 & 11.603$\pm$1.271 & 58.007$\pm$4.895 & 26.014$\pm$6.423 & 39.728$\pm$0.905 & 53.630$\pm$4.183 \\\hline
\rowcolor{gray!40}
\multicolumn{8}{|c|}{Sparse Ratio ($\|\bm w_t \|_0 / d$)} \\\hline 
farmads & \textbf{.0130$\pm$.0059} & .2903$\pm$.0541 & .8099$\pm$.0203 & .3072$\pm$.0101 & 1.0000$\pm$.0000 & .1933$\pm$.0137 & .8095$\pm$.0202 \\\hline
 real-sim & \textbf{.3236$\pm$.0275} & 1.0000$\pm$.0000 & 1.0000$\pm$.0000 & .9666$\pm$.0301 & 1.0000$\pm$.0000 & .9792$\pm$.0571 & 1.0000$\pm$.0000 \\\hline 
 rcv1b & \textbf{.2987$\pm$.0273} & .8398$\pm$.1175 & .9571$\pm$.0010 & .9054$\pm$.0629 & 1.0000$\pm$.0000 & .7683$\pm$.0193 & .9571$\pm$.0010  \\\hline 
 news20b &  \textbf{.0016$\pm$.0004} & .9224$\pm$.0037 & .9224$\pm$.0037 & .9224$\pm$.0037 & 1.0000$\pm$.0000 & .9220$\pm$.0039 & .9208$\pm$.0039 \\\hline 
reviews & \textbf{.0006$\pm$.0002} & .6622$\pm$.1976 & .7310$\pm$.0051 & .6618$\pm$.1977 & 1.0000$\pm$.0000 & .5155$\pm$.2409 & .7292$\pm$.0064 \\\hline
imdb & \textbf{.0320$\pm$.0081} & .8661$\pm$.0009 & .8661$\pm$.0009 & .6083$\pm$.3940 & 1.0000$\pm$.0000 & .8661$\pm$.0009 & .8661$\pm$.0009 \\\hline
\rowcolor{gray!40}
 \multicolumn{8}{|c|}{AUC score} \\\hline 
farm-ads & \underline{.94290$\pm$.00669}\tnote{4} & .92486$\pm$.00631 & .92914$\pm$.01086 & .92609$\pm$.00599 & .94210$\pm$.00899 & \textbf{.95402$\pm$.00511} & .95212$\pm$.00691 \\\hline
 real-sim & \underline{.99226$\pm$.00081} & .98541$\pm$.00115 & .98741$\pm$.00111 & .97911$\pm$.00532 & .99141$\pm$.00079 & \textbf{.99394$\pm$.00058} & .99331$\pm$.00059 \\\hline 
 rcv1b & \underline{.99488$\pm$.00017} & .99258$\pm$.00074 & .99365$\pm$.00031 & .99257$\pm$.00081 & .99358$\pm$.00016 & .99553$\pm$.00013 & \textbf{.99544$\pm$.0001} \\\hline 
 news20b  & \underline{.97434$\pm$.00193} & .94420$\pm$.02221 & .96008$\pm$.00796 & .94906$\pm$.02245 & .95411$\pm$.00228 & \textbf{.99240$\pm$.00104} & .98260$\pm$.00247 \\\hline 
reviews & \underline{.91320$\pm$.00840} & .88822$\pm$.01426 & .89825$\pm$.01426 & .89407$\pm$.01304 & .90799$\pm$.01220 & \textbf{.93343$\pm$.00915} & .91826$\pm$.00592 \\\hline
imdb & \underline{.94614$\pm$.00275} & .86767$\pm$.01199 & .86911$\pm$.01198 & .80066$\pm$.10940 & .89735$\pm$.00541 & \textbf{.94983$\pm$.00223} & .92644$\pm$.00520 \\\hline
\end{tabular}
\begin{tablenotes}
\item[1] All reported values are averaged on 10 trials of outcomes by randomly shuffling all training samples.
\item[2] One should notice that \textsc{FSAUC} is not a true online learning algorithm. It needs to have the total number of input samples as a prior. We treat it as an offline algorithm as a reference and do not attend to compare it with other online methods.
\item[3] The run time of a specific method is the total running time of passing whole training samples once.
\item[4] Underline of the AUC score means that it is the runner-up among all online methods (excluding the offline method \textsc{FSAUC}).
\end{tablenotes}
\end{threeparttable}%
\label{table:auc-performance}
\end{table*}

\subsection{Model sparsity and AUC performance}
\textbf{Q2} is affirmatively addressed in the following two parts.

\noindent\textbf{Model sparsity.\quad} As discussed, the second improvement is that our algorithm can capture sparsity more effectively. To answer the first part of ~\textbf{Q2}, we measure the sparsity of final model $\bm w_t$ by using the sparse ratio which is defined as the following
\begin{equation}
\operatorname{Sparse\ Ratio} := \frac{\|\bm w_t\|_0}{d}. \nonumber
\end{equation}
Without sacrificing the AUC score, the sparser the model, the better. The middle section of Table~\ref{table:auc-performance} shows clearly that the sparse ratio of the models obtained by \textsc{FTRL-AUC} is much sparser compared with other methods. For example, in imdb dataset, \textsc{FTRL-AUC} needs only about 2,865 features on average to get AUC score 0.94614, while \textsc{SPAUC}, the method getting the best AUC score 0.94983 needs to have up to 77,539 features on average. Clearly, this makes the $\ell^1$-regularization less meaningful for \textsc{SPAUC}. The sparse ratio of \textsc{SOLAM} is always 1.0 because at the beginning, it needs $\bm w_0$ on the ball $\{\bm w: \|\bm w\|\leq R\}$. The excellent sparsity pursuing ability of our method can also been seen via the model selection phase which will be discussed in the later part of this section.

\noindent\textbf{AUC score.\quad} We compare the AUC score of $\bm w_t$ on testing data of all methods. The bottom section of Table~\ref{table:auc-performance} presents the AUC scores. The up-to-now best performance of AUC score is \textsc{SPAUC}~\cite{ying2019spauc} in most cases, a recently developed AUC optimization method. One of the key properties of \textsc{SPAUC} is that each objective function is convex compared with \textsc{SPAM}-based methods. Another advantage of \textsc{SPAUC} is based on a stochastic proximal update step developed in~\cite{duchi2009efficient}. From the empirical evaluation point of view, a possible reason that \textsc{SPAUC} achieves the best performance on AUC score is because it converges faster than these baselines: the learning rate of \textsc{SPAUC} is $\mathcal{O}(1/t)$ while others are $\mathcal{O}(1/\sqrt{t})$. 

By contrast, our method is consistently the runner-up among all online methods. It means that our method is competitive with respect to AUC score. For example, in real-sim, rcv1b, and imdb, the AUC scores of our method are only, 0.169\%, 0.065\%, and 0.388\% less than these of \textsc{SPAUC} respectively. In new20b dataset, our method is 2.167\% less than \textsc{SPAUC}. Again, this is because our method uses $\mathcal{O}(1/\sqrt{t})$ learning rate which is slower than $\mathcal{O}(1/t)$ updates of \textsc{SPAUC}. Furthermore, by pursuing sparse solution, our method threshold out some less frequent  but important nonzero features. However, \textsc{SPAUC} hardly obtain sparse solution by using an $\ell^1$ regularization and has much slower run time. This indicates that our method could be a good alternative if real-world application needs sparse solution due to memory or run time consideration.

\begin{figure}[ht!]
\centering
\includegraphics[width=8.8cm,height=6cm]{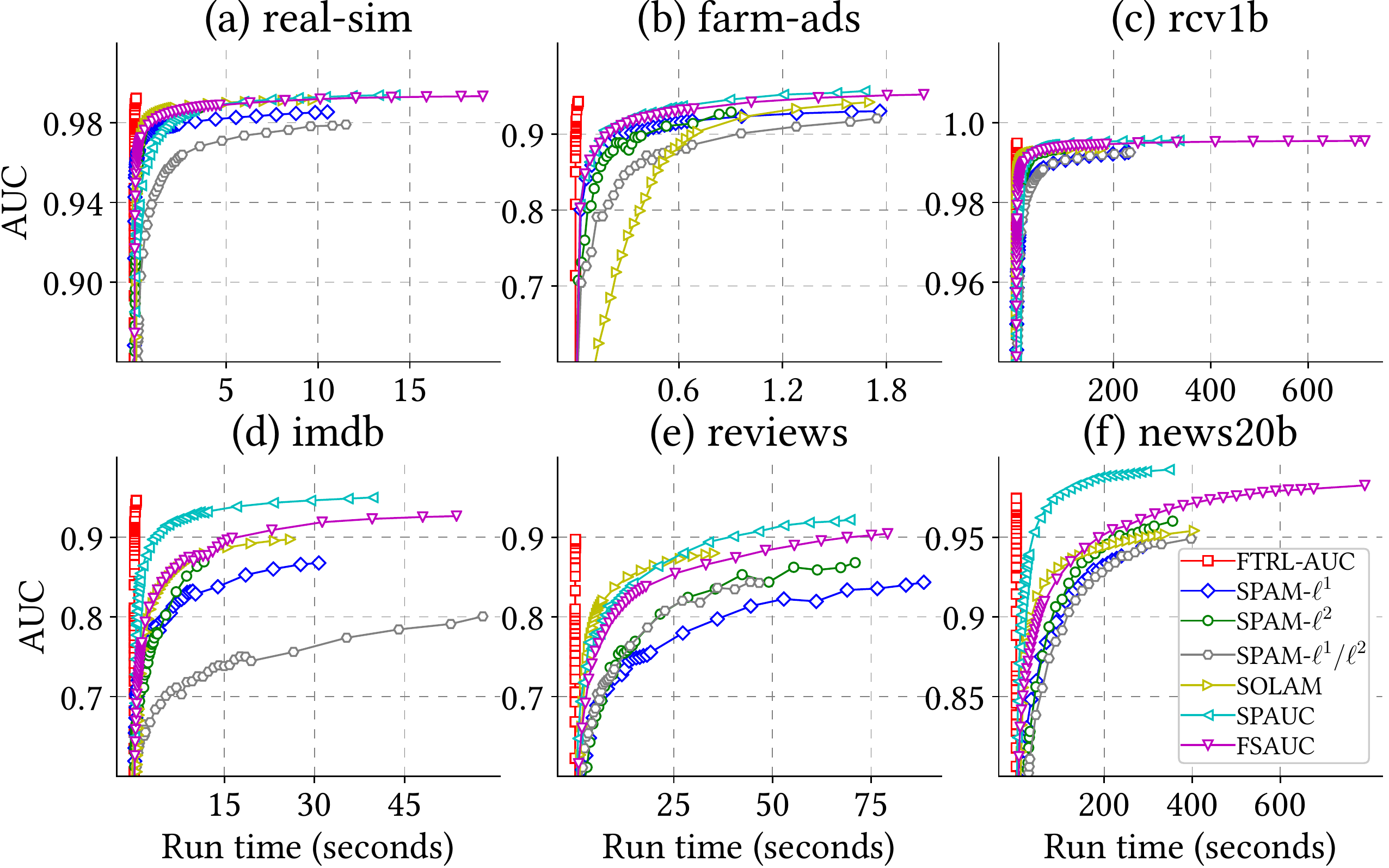}
\caption{Convergence Rate with respect to run time\vspace{-2mm}}
\label{fig:run-time}
\end{figure}

\begin{figure}[ht]
\centering
\includegraphics[width=8.8cm,height=5.5cm]{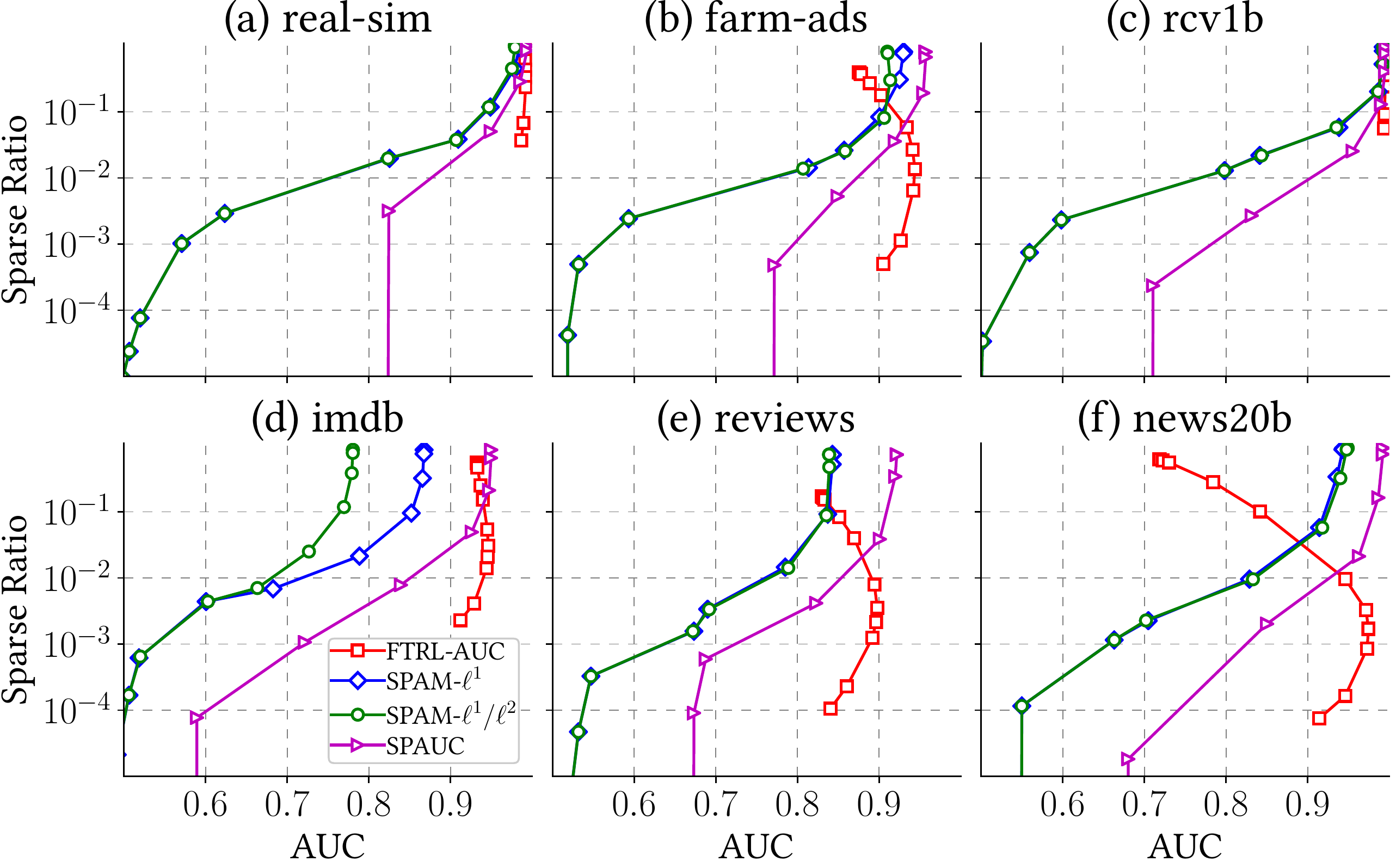}
\caption{Sparse Ratio as a function of the AUC score\vspace{-2mm}}
\label{fig:sparsity}
\end{figure}

The convergence with respect to the run time is illustrated in Figure~\ref{fig:run-time}. This figure clearly demonstrates that our method converges extremely fast over time. The convergence with respect to the number of training samples can be found in the appendix.

\noindent\textbf{Sparsity tuning.\quad} To have a fair comparison, we choose $\lambda$, the $\ell^1$ regularization parameter, from a large range set $\{10^{-8}, 10^{-7}$, $10^{-6}, 10^{-5}, 10^{-4}$, $10^{-3}, 0.005, 0.01, 0.05, 0.1, 0.3, 0.5, 0.7, 1.0, 3.0, 5.0\}$ for \textsc{FTRL-AUC}, \textsc{SPAM}-$\ell^1$, \textsc{SPAM}-$\ell^1/\ell^2$, and \textsc{SPAUC}. We want to see if there is any AUC gain to pursue sparsity. Surprisingly, all baselines of online AUC optimization fail to gain AUC score when pursuing the sparsity. The experimental results are illustrated in Figure~\ref{fig:sparsity}. Clearly, \textsc{FTRL-AUC} boosts the AUC score when it tries the $\lambda$ from lowest $10^{-8}$ to high. However, for the other three methods, the AUC scores decrease dramatically when $\lambda$ increases. One possible explanations is that the ability of FTRL tries to use all previous gradient information $\sum_{t=1}^T \nabla \widehat{f}_t(\bm w_t)$ while all the other methods only use current gradient information. To approximate the true gradient, the accumulated gradients are clearly more stable and effective than the single gradient.

\begin{figure}[ht]
\centering
\includegraphics[width=8.8cm,height=5.5cm]{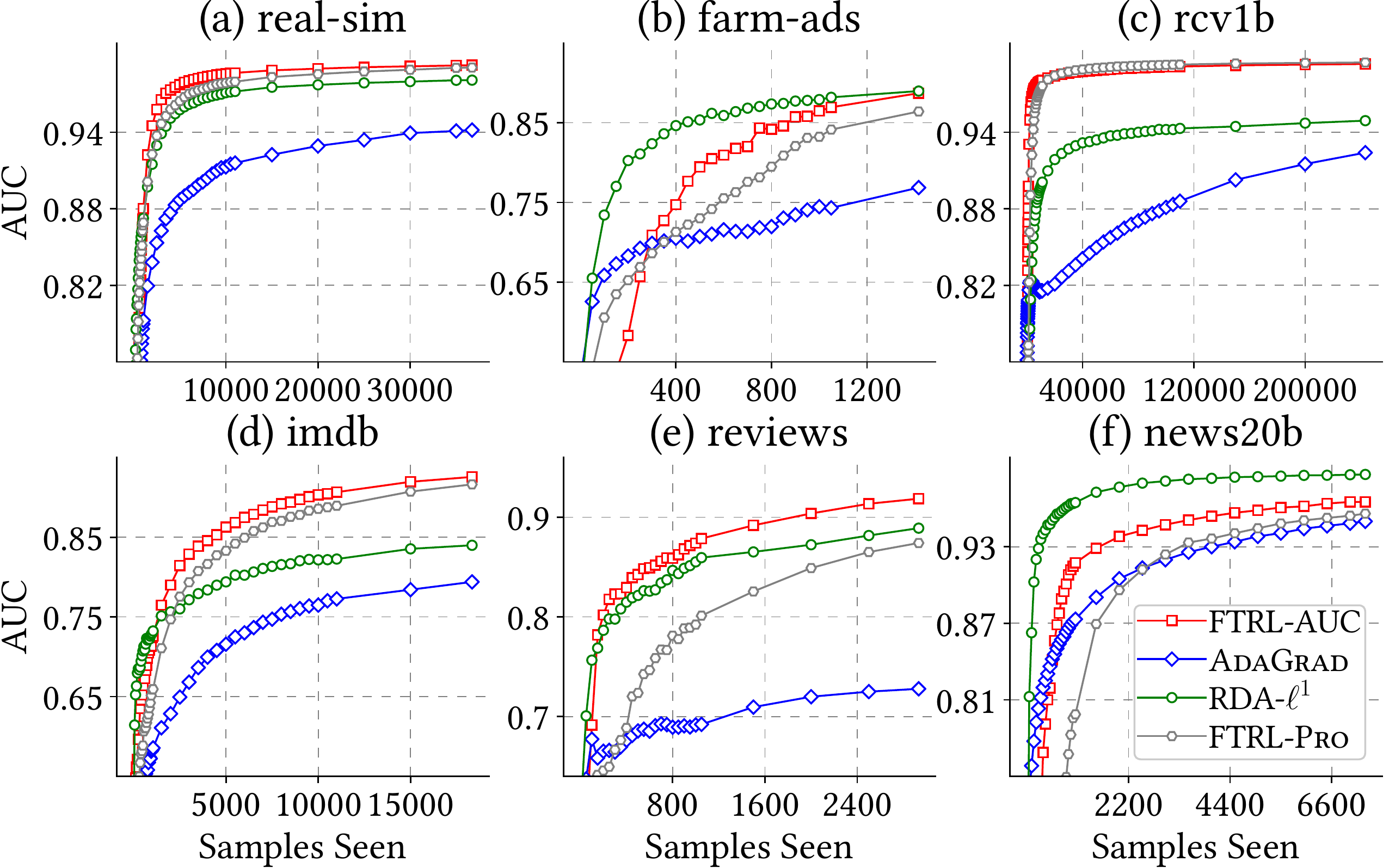}
\caption{AUC as function of the number of samples seen}
\label{fig:curves-imbalance-iter}
\end{figure}

\begin{table}[ht!]
\renewcommand{\tabcolsep}{3pt}
\caption{Comparison of AUC score with logistic loss methods}
\centering
 \begin{threeparttable}[t]
\begin{tabular}{|P{0.065\textwidth}|P{0.09\textwidth}|P{0.08\textwidth}|P{0.09\textwidth}|P{0.09\textwidth}|}\hline
Datasets & FTRL-AUC & AdaGrad & RDA-$\ell^1$ & FTRL-Pro \\\hline 
\rowcolor{gray!40}
\multicolumn{5}{|c|}{AUC Score $(T_+/T_-=0.1)$} \\\hline
farm-ads & .8870$\pm$.0169 & .7686$\pm$.0487 & \textbf{.8897$\pm$.0335} & .8641$\pm$.0203 \\\hline
 real-sim & \textbf{.9926$\pm$.0007} & .9417$\pm$.0046 & .9809$\pm$.0018 & .9907$\pm$.0012 \\\hline 
 rcv1b & .9935$\pm$.0004 & .9240$\pm$.0023 & .9491$\pm$.0018 & \textbf{.9947$\pm$.0003} \\\hline 
 news20b  & .9752$\pm$.0054 & .9599$\pm$.0061 & \textbf{.9967$\pm$.0010} & .9660$\pm$.0062 \\\hline 
reviews &  \textbf{.9187$\pm$.0130} & .7278$\pm$.0341 & .8891$\pm$.0161 & .8742$\pm$.0088 \\\hline
imdb &  \textbf{.9256$\pm$.0055} & .7939$\pm$.0164 & .8399$\pm$.0178 & .9163$\pm$.0050 \\\hline
\rowcolor{gray!40}
\multicolumn{5}{|c|}{AUC Score $(T_+/T_-=0.05)$} \\\hline
farm-ads & .8723$\pm$.0343 & .6697$\pm$.1260 & \textbf{.8784$\pm$.0331} & .8076$\pm$.0515 \\\hline
 real-sim & \textbf{.9939$\pm$.0007} & .9413$\pm$.0066 & .9789$\pm$.0031 & .9899$\pm$.0014 \\\hline 
 rcv1b & .9925$\pm$.0003 & .8997$\pm$.0039 & .9439$\pm$.0019 & \textbf{.9938$\pm$.0006} \\\hline 
 news20b  & .9638$\pm$.0120  & .9490$\pm$.0173 & \textbf{.9969$\pm$.0017} & .9541$\pm$.0140 \\\hline 
reviews &  \textbf{.8781$\pm$.0291} & .6584$\pm$.0505 & .8746$\pm$.0236  & .8260$\pm$.0297 \\\hline
imdb &  \textbf{.9133$\pm$.0152}  & .8159$\pm$.0262 & .8412$\pm$.0175 & .9021$\pm$.0163 \\\hline
\end{tabular}
\end{threeparttable}%
\label{tab:imbalanced}
\end{table}

\subsection{Performance on imbalanced datasets}

As we have mentioned, our method could better optimize the AUC score when the dataset is imbalanced. To answer \textbf{Q3}, we compare \textsc{FTRL-AUC} with the online learning algorithms for logistic loss including \textsc{RDA}-$\ell^1$, \textsc{AdaGrad}, and \textsc{FTRL-Pro}. To make these six high-dimensional sparse datasets imbalanced, we use the following strategy: In our first experiment, we keep all negative training samples ($T_-$ in total) and only keep the first $0.1 * T_-$ positive samples so that the imbalance ratio  is low $T_+/T_- = 0.1$. Similarly, in our second experiment, we only keep the first $0.05*T_-$ positive samples. The rest experimental setup remains unchanged. 

The AUC scores of these two experiments are reported in Table~\ref{tab:imbalanced}.  Compared with \textsc{FTRL-Pro}\footnote{The only difference between our method and \textsc{FTRL-Pro} is that \textsc{FTRL-Pro} uses logistic loss but our method uses AUC loss defined in~(\ref{inequ:auc-objective}).}, our method have much higher AUC score on farm-ads, reviews. The best performance of \textsc{FTRL-Pro} on rcv1b is 0.9947 while \textsc{RDA}-$\ell^1$ obtains the best performance on farm-ads and news20b. For rcv1b dataset, our method is only 0.12\% less than the best one, still competitive. From Figure~\ref{fig:curves-imbalance-iter} ($T_+/T=0.1$), we can see that the convergence curve of \textsc{FTRL-AUC} than that of \textsc{FTRL-Pro}, indicating that advantage of our method. One may notice that \textsc{RDA}-$\ell^1$ is much better than ours on news20b dataset. This may be because \textsc{RDA}-$\ell^1$ but it performs much worse than ours on imdb, reviews, and rcv1b datasets. Figure~\ref{fig:para-select-imblance} presents the parameter tuning of $\lambda$ of these four methods.

\begin{figure}[ht!]
\centering
\includegraphics[width=8.8cm,height=5.5cm]{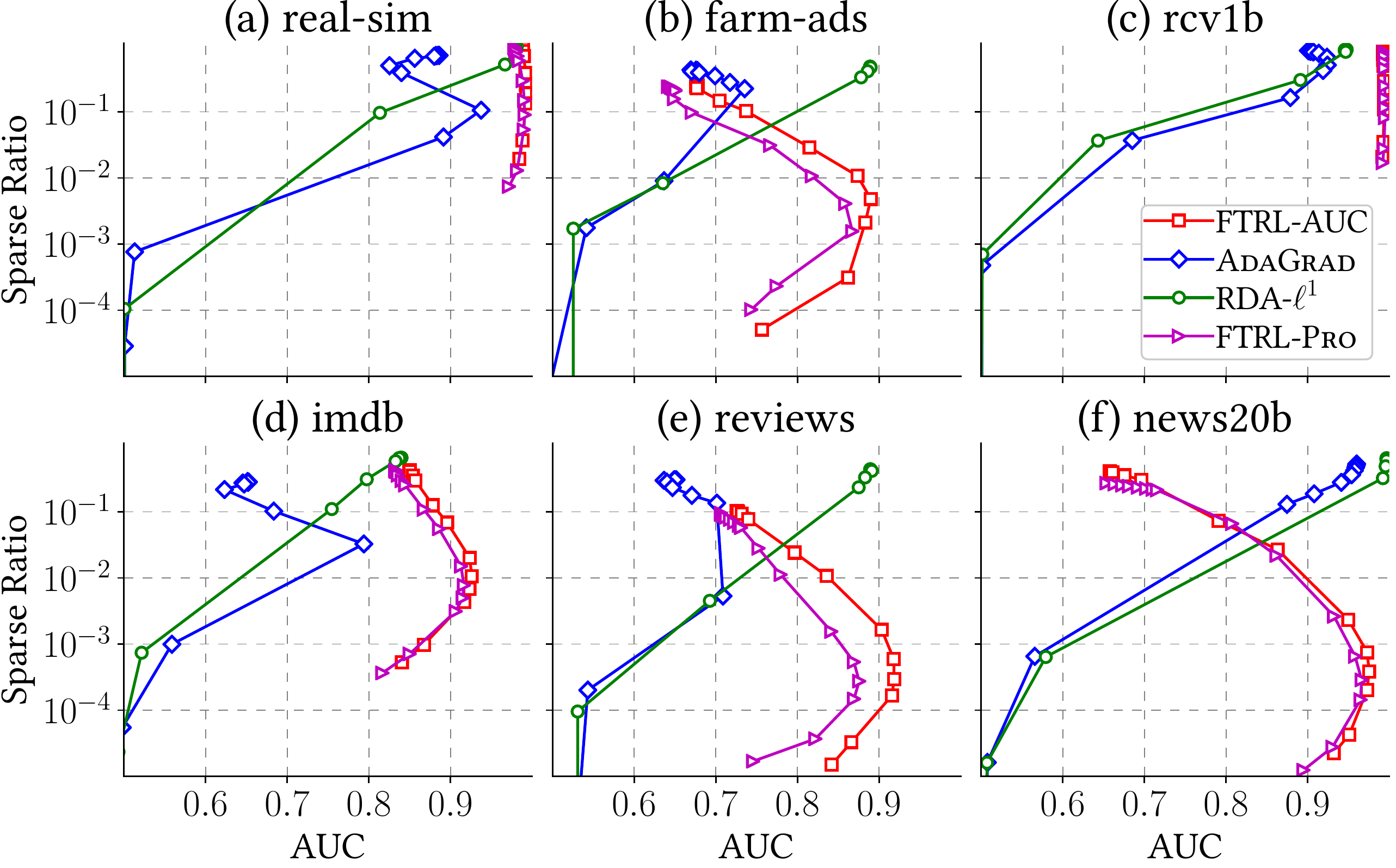}
\caption{Sparse Ratio as a function of AUC score}
\label{fig:para-select-imblance}
\end{figure}

\begin{figure}[ht!]
\centering
\includegraphics[width=7cm,height=5cm]{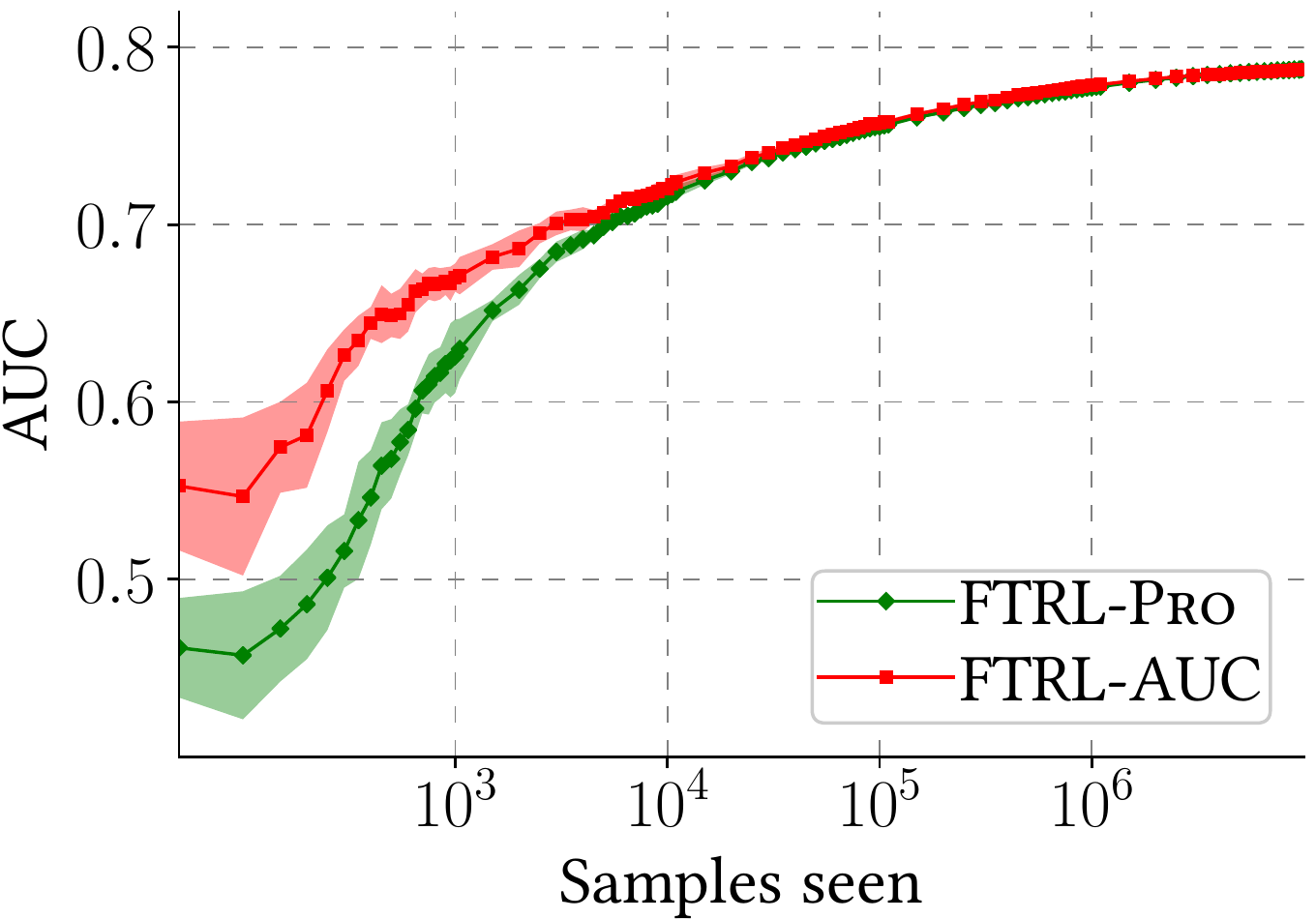}
\caption{The convergence curve with respect to the number of training samples seen on Avazu dataset}
\label{fig:avazu}
\end{figure}

\noindent\textbf{Avazu: Click‐through rate prediction.\quad} The task of online click-through rate prediction is to predict whether a user will click an advertisement or not. We compare our method with \textsc{FTRL-Pro} on avazu of the Kaggle dataset where the original dataset is released at \url{https://www.kaggle.com/c/avazu-ctr-prediction}, which has about 14 million samples. In the feature engineering step, we use the data preprocessing step, a field-aware factorization machines proposed in~\cite{ffm2016}. Each training sample has 1 million features after the factorization processing. Again, we split them as training, validation, and testing samples by 4:1:1. Both methods have the same parameter space. We randomly shuffle the dataset 10 times and run both methods and report the convergence curve in Figure~\ref{fig:avazu}. When the two algorithms receive enough training samples, the AUC scores do not have much difference. However, at the early stage of the learning process, \textsc{FTRL-AUC} achieves significantly higher AUC scores than \textsc{FTRL-Pro}.

\section{Conclusion and Future Work}
\label{sect:conclusion}
To conclude, in this paper, we propose a faster online AUC optimization algorithm based on a generalized follow-the-regularized-leader framework. By using a new ``lazy update'' formula, we reduce the per-iteration time cost from $\mathcal{O}(d)$ to $\mathcal{O}(k)$. Our experimental results demonstrate that \textsc{FTRL-AUC} can significantly reduce the run time as well as obtain sparse models more effectively. This makes our method attractive to very high-dimensional sparse datasets. For the future work, an interesting research question is that what is the error between the proposed loss and the true loss and how to control it. Furthermore, we can try to improve the AUC score performance without loss of the run time and model sparsity advantages. One potential direction is to explore how to incorporate the second order information. Also, it remains interesting to see if our work can be generalized to multi-classification tasks.

\section{Acknowledgement}

The authors would like to thank anonymous reviewers for their helpful comments on the paper. The work of Yiming Ying is supported by NSF grants IIS-1816227 and IIS-2008532. This work was partially supported by NSF grants IIS-1926751, IIS-1927227, IIS-1546113, and the New York Empire Innovation Program.

\bibliographystyle{unsrt}
\bibliography{references}

\appendix

\subsection{Reproducibility and detailed experimental Setup}
\label{append:reprod}
\textbf{Implementation Details.\quad} To reproduce results including results of baselines, we present implementation details as follows:
\begin{itemize}[leftmargin=*]
\item All methods are implemented in C11, a C standard revision, language with Python2.7 as a wrapper. The experiments are executed in a cluster with 5 nodes. Each node has 28 CPUs and 250Gb memory. For each method, we only use 1 CPU at a time.
\item The random seeds for all trials are np.random.seed(17), which makes results of AUC scores and sparse ratios reproducible.
\item Critical operations of all baseline methods are scale product $c\cdot \bm x$ and the inner product $\langle\bm x, \bm y \rangle=\bm x^\top \bm y$, which are calculated by $\text{cblas}\_\text{dscal}()$ and $\text{cblas}\_\text{ddot}()$ respectively. These two functions are provided by OpenBLAS~\cite{xianyi2012model}~\footnote{\url{https://github.com/xianyi/OpenBLAS} with version 0.3.1 (Accessed in February 2020)}, an optimized BLAS library.

\item For \textsc{SPAM}-$\ell^1$, \textsc{SPAM}-$\ell^2$, and \textsc{SPAM}-$\ell^1/\ell^2$, since they need to estimate $\widehat{p}_T, a(\bm w_t), b(\bm w_t)$, and $\alpha(\bm w_t)$, in our experiments, we estimate them by using $p_t, a_t(\bm w_t), b_t(\bm w_t),$ and $\alpha_t(\bm w_t)$ defined in~(\ref{inequ:new-ft}).
\item For \textsc{FSAUC}, there exists a projection step onto a $\ell^1$-norm ball. The projection used in the original implementation is the method proposed in~\cite{duchi2008}. However, there exists a much faster version of $\ell^1$-ball projection~\cite{condat2016fast} as claimed has $\mathcal{O}(d)$ run time in practice~\footnote{The C version code can be download from \url{https://lcondat.github.io/download/condat_l1ballproj.c (Accessed in February 2020)}.}.
\end{itemize}

\begin{figure}[ht]
\centering
\includegraphics[width=8cm,height=5cm]{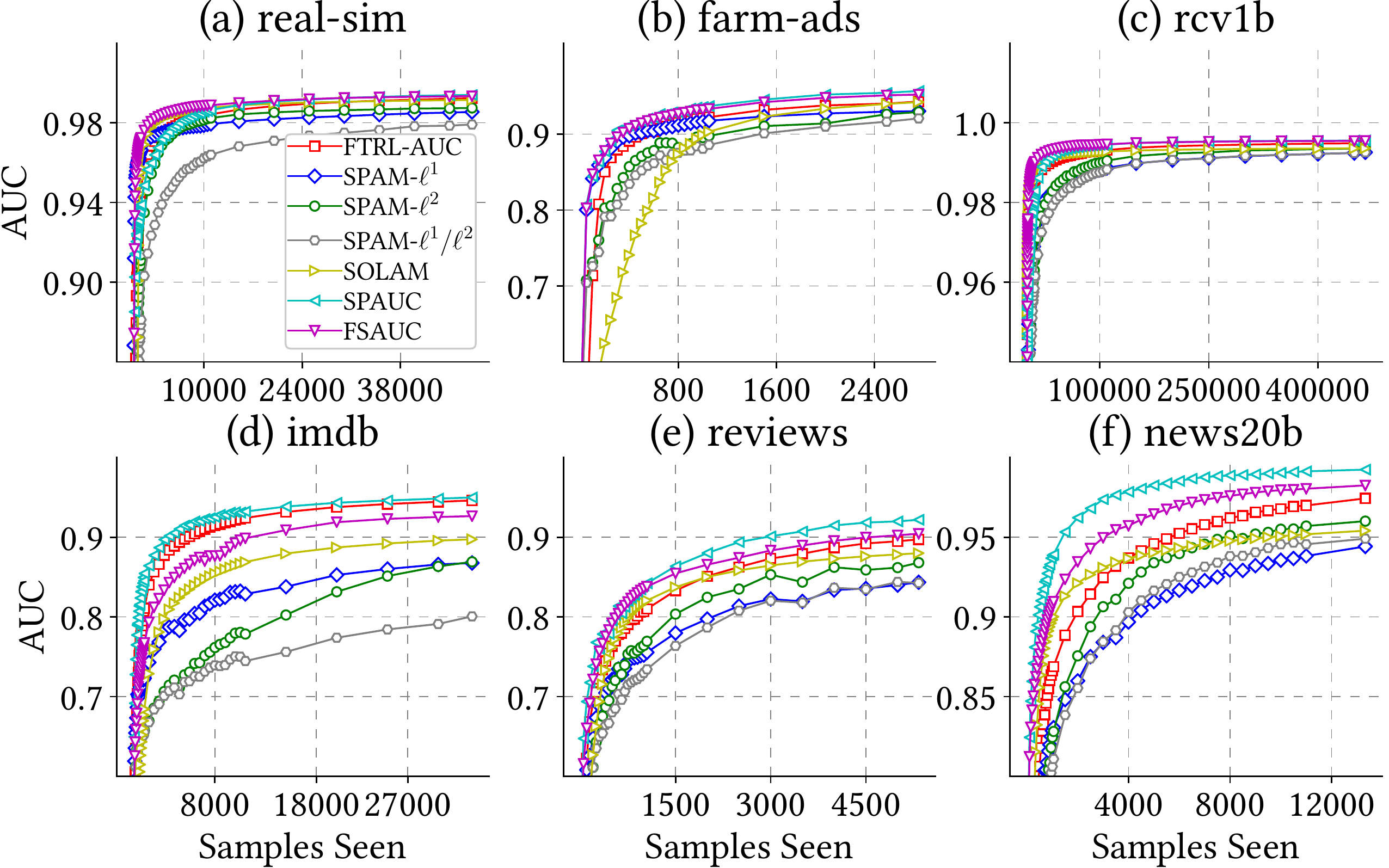}
\caption{Convergence rate with respect to the number of training samples seen}
\label{fig:all-converge-curve-iter}
\end{figure}

\noindent\textbf{Parameter Tuning.\quad} We list parameter tuning of all methods including the baseline methods as follows:
\begin{itemize}[leftmargin=*]
\item \textsc{FTRL-AUC}, it has two parameters. The $\ell^1$-regularization parameter $\lambda$ which is from a sufficient large range $\{10^{-8}, 10^{-7}$,\ldots, $10^{-3}$, 0.005, 0.01, 0.05, 0.1, 0.3, 0.5, 0.7, 1.0, 3.0, 5.0$\}$, and the initial learning rate $\gamma$ is from $\{10^{-5}, 5\cdot 10^{-5}$, 0.0001, 0.0005, 0.001, 0.005, 0.01, 0.05, 0.1, 0.5, 1.0, 5.0$\}$.
\item \textsc{SPAM}-$\ell^1$ has two parameters. The initial learning rate $\xi$ is from $\{10^{-3},10^{-2},10^{-1},10^{0},10^{1},10^{2},10^{3}\}$. The $\ell^1$-regularization parameter is the same as \textsc{FTRL-AUC}'s.
\item \sloppy\textsc{SPAM}-$\ell^2$ has two parameters. The initial learning rate $\xi$ is from $\{10^{-3},10^{-2},10^{-1},10^{0},10^{1},10^{2},10^{3}\}$. The $\ell^2$-regularization parameter is the same as $\lambda$.
\item \sloppy \textsc{SPAM}-$\ell^1/\ell^2$ has 3 main parameters. To avoid large cross-validation time, the parameter $\xi$ and $\lambda^2$ is used by the parameter tuned from \textsc{SPAM}-$\ell^2$. We only tune the $\ell^1$ parameter $\lambda_1$ which is the same $\lambda$.
\item \textsc{FSAUC} has 2 parameters. The $\ell^1$-norm ball which is from $\{10^{-1}, 10^{0},\ldots,10^{5}\}$. The corresponding initial learning rate is from $\{2^{-10},2^{-9},2^{-8},\ldots,2^{8},2^{9},2^{10}\}$ as suggested in~\cite{liu2018fast}.
\item \textsc{SOLAM} has two parameters. The $\ell^2$-norm ball diameter which is from $\{10^{-1}, 10^{0},\ldots,10^{5}\}$ and the initial learning rate $\xi \in \{1.0,10.0,19.0,28.0,\ldots,100.0\}$ as suggested in~\cite{ying2016stochastic}.
\item \textsc{SPAUC} has two parameters. The initial learning rate parameter is from $\{10^{-7.0},10^{-6.5},10^{-6.0},\ldots,10^{-2.5}\}$. Since we use the $\ell^1$-regularization and it is the same $\lambda$.
\item \textsc{FTRL-Pro} has the same parameter tuning strategy as \textsc{FTRL-AUC}.
\item \textsc{RDA}-$\ell^1$ has three parameters. It has an initial learning rate from the range $\{10.0, 50.0, 100.0, 500.0, 1000.0, 5000.0\}$. The sparsity-enhancing parameter is from $\{0.0,0.005\}$, where 0.0 corresponding to non-enhancing sparsity. The $\lambda$ is the same as \textsc{FTRL-AUC}'s.
\item \sloppy \textsc{AdaGrad} has three parameters. The $\epsilon$ is fixed to $10^{-8}$ to avoid the divided by zero error. The learning rate parameter is from $\{0.001, 0.01, 0.1, 1.0, 10.0, 50.0, 100.0, 500.0, 1000.0, 5000.0\}$ while $\lambda$ is the same as others.
\end{itemize}

\subsection{More Results}
\label{append:more-results}

We first provide the experimental details for Figure~\ref{fig:compare-update-rule}. We fix the initial learning rate $\gamma = 1.0$ and the sparsity parameter $\lambda = 0.5$. The convergence curves illustrate in Figure~\ref{fig:compare-update-rule} and~\ref{fig:compare-update-rule-2} are the AUC scores averaged on 10 trials. 

\begin{figure}[ht]
\centering
\includegraphics[width=8.52cm,height=3.51cm]{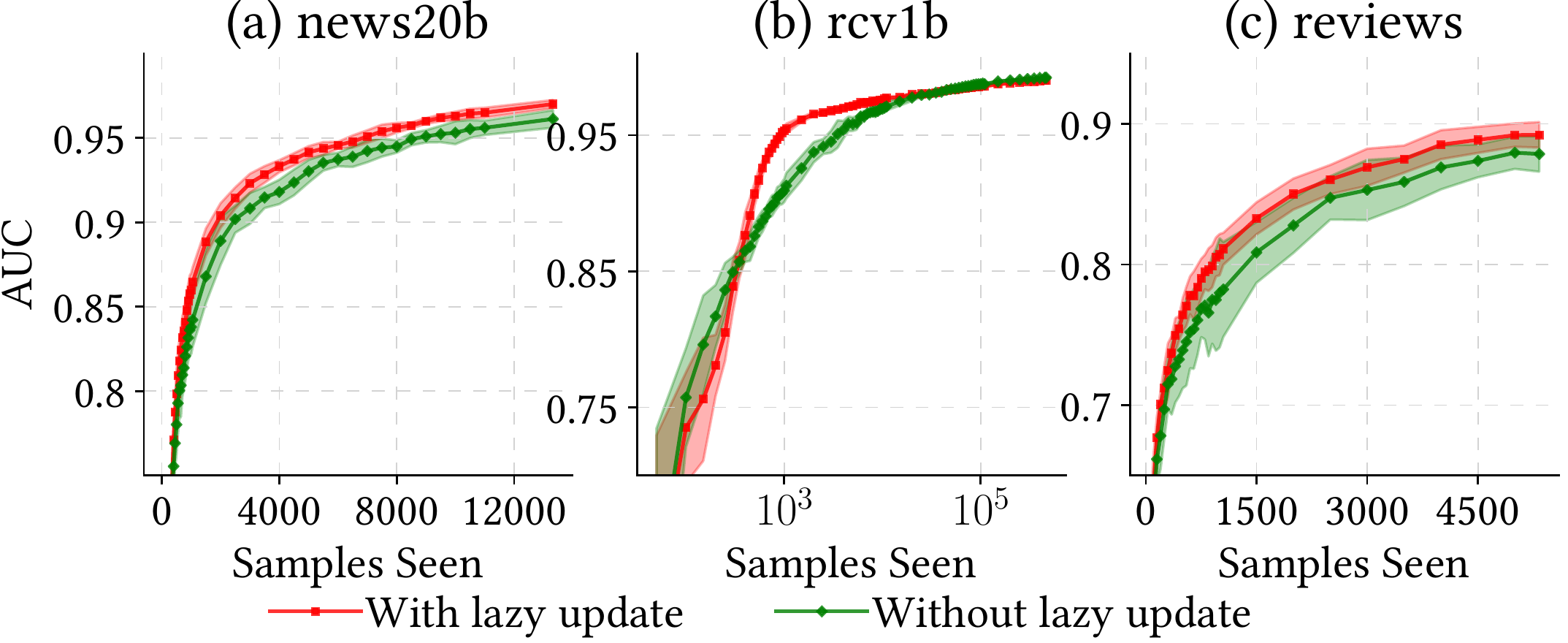}
\caption{Comparison of with and without ``lazy update''}
\label{fig:compare-update-rule-2}
\end{figure}

The convergence curve with respect to the number training samples seen are illustrated in Figure~\ref{fig:all-converge-curve-iter}. In general, the performance of \textsc{FTRL-AUC} on the convergence is better than \textsc{SPAM}-based. Figure~\ref{fig:10} illustrates the convergence curve as a function the number of training samples seen for the datasets of imbalance ratio $T_+ / T = 0.05$. Figure~\ref{fig:11} illustrate the sparse ratio and corresponding AUC scores as a function of the parameter $\lambda$ for the datasets of imbalance ratio $T_+ / T = 0.05$.

\begin{figure}[!ht]
\centering
\includegraphics[width=8cm,height=5cm]{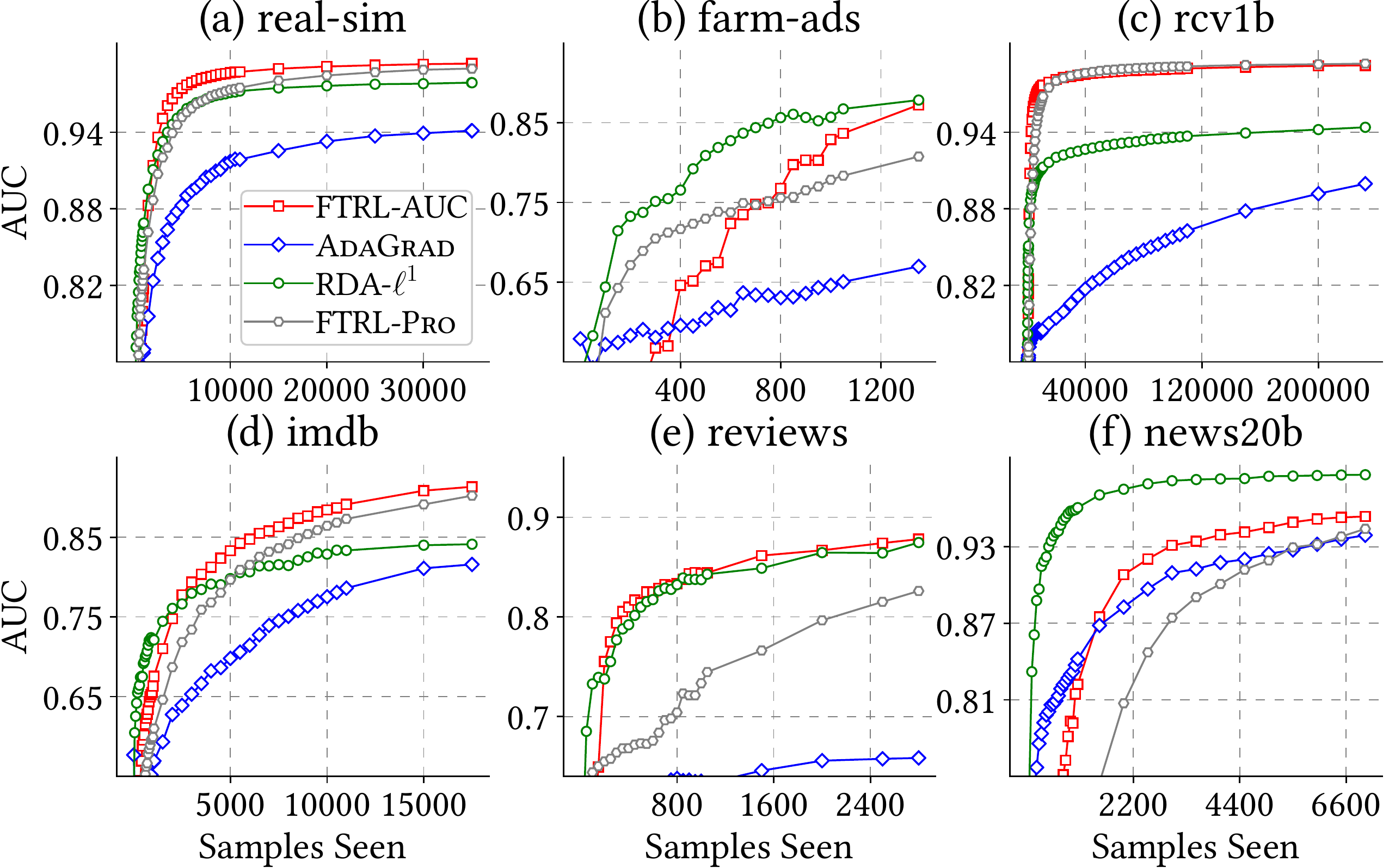}
\caption{Convergence curve with respect to the number of training samples seen}
\label{fig:10}
\end{figure}

\begin{figure}[!ht]
\centering
\includegraphics[width=8cm,height=5cm]{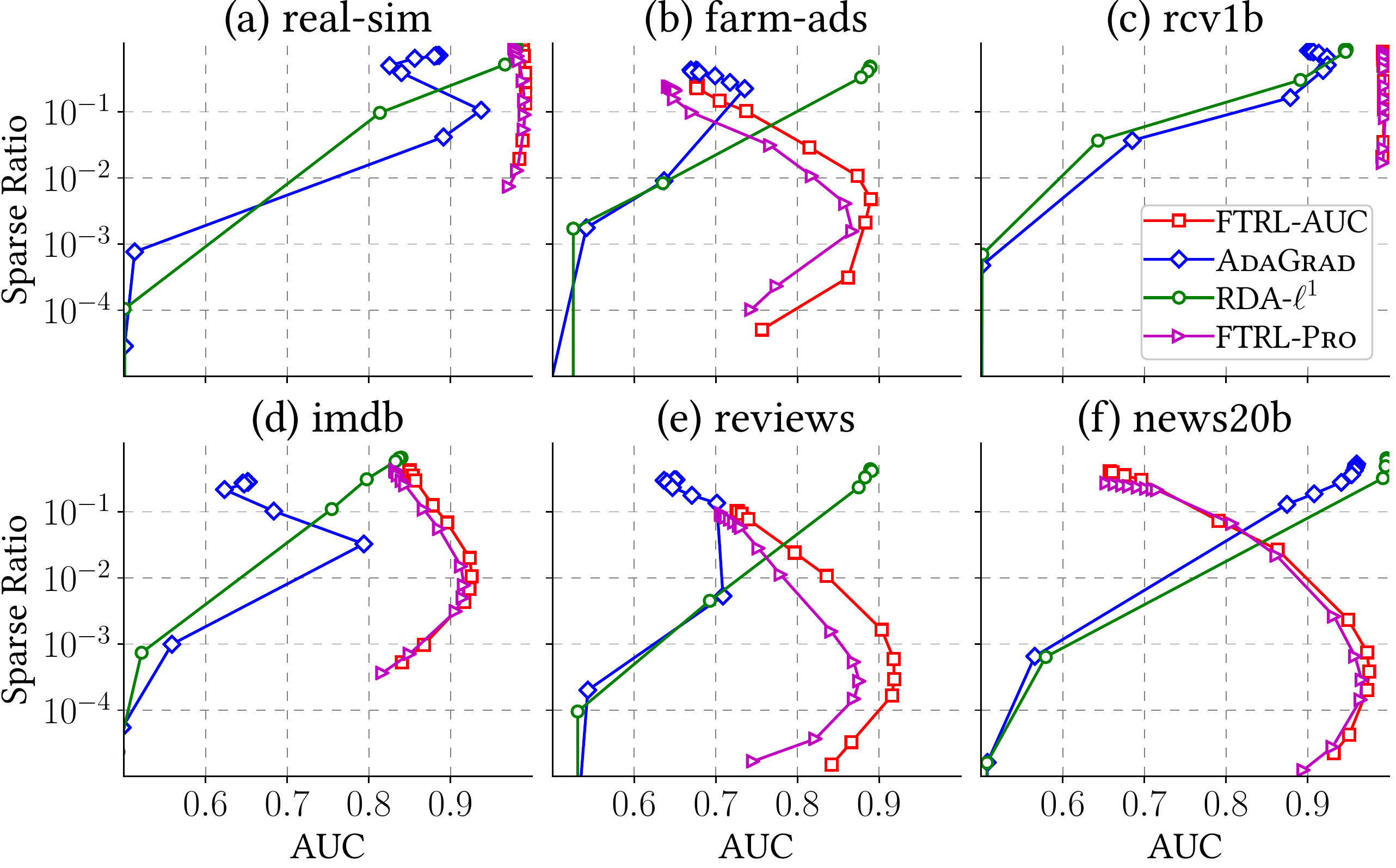}
\caption{Sparse Ratio as a function of the AUC score.}
\label{fig:11}
\end{figure}
\end{document}